%% file: main.tex
\documentclass[letterpaper]{article} 
\usepackage{aaai2026}  
\usepackage{times}  
\usepackage{helvet}  
\usepackage{courier}  
\usepackage[hyphens]{url}  
\usepackage{graphicx} 
\usepackage{natbib}  
\usepackage{caption} 
\usepackage{algorithm}
\usepackage{algorithmic}

\usepackage{newfloat}
\usepackage{listings}
\DeclareCaptionStyle{ruled}{labelfont=normalfont,labelsep=colon,strut=off} 
\floatstyle{ruled}
\newfloat{listing}{tb}{lst}{}
\floatname{listing}{Listing}
\usepackage{hyperref} 
\nocopyright 

\usepackage{mathrsfs}
\usepackage{amsmath}
\usepackage{amsthm} 
\usepackage{amssymb}
\theoremstyle{plain}
\newtheorem{theorem}{Theorem}
\newtheorem{lemma}{Lemma}
\newtheorem{proposition}{Proposition}
\newtheorem{corollary}{Corollary}
\usepackage{booktabs}            
\newcommand{\cmark}{\checkmark}  
\newcommand{\xmark}{\(\times\)}  
\makeatletter
\newcommand{\STATEx}[1]{\STATE \hskip 16\algorithmicindent #1}
\makeatother

\newcommand{\tuple}[1]{\ensuremath{\left \langle #1 \right \rangle }}
\newcommand{\eps}{\varepsilon}
\newcommand{\powsetopname}{\mathscr{P}}
\newcommand{\powset}[1]{\powsetopname(#1)}
\newcommand{\prob}{\mathbb{P}}
\newcommand{\A}{\mathbf{A}}
\newcommand{\Rob}{\mathbf{R}}
\newcommand{\R}{\mathbb{R}}
\newcommand{\Nat}{\mathbb{N}}
\newcommand{\N}{[N]}
\newcommand{\M}{[M]}
\newcommand{\CP}{\mathbb{C}\mathbb{P}}

\newcommand{\agree}{\ensuremath{\tuple{M, N, \eps, \delta}}}
\newcommand{\E}[1]{\mathbb{E}\left[#1\right]}

\newcommand{\Eja}[1]{\mathbb{E}_{{\prob}^{\A}_j}\left[#1\right]}
\newcommand{\Ejr}[1]{\mathbb{E}_{{\prob}^{\Rob}_j}\left[#1\right]}
\newcommand{\Eji}[1]{\mathbb{E}_{{\prob}^{i}_j}\left[#1\right]}
\newcommand{\Ejk}[1]{\mathbb{E}_{{\prob}^{k}_j}\left[#1\right]}

\theoremstyle{definition}
  \newtheorem{definition}{Definition}
  
  \newtheorem{requirement}{Requirement}

\theoremstyle{definition}

\title{Intrinsic Barriers and Practical Pathways for Human-AI Alignment: An Agreement-Based Complexity Analysis}
\author {
    Aran Nayebi
}
\affiliations{
    Machine Learning Department and Neuroscience \& Robotics Institutes,\\
    School of Computer Science, Carnegie Mellon University\\
    anayebi@cs.cmu.edu
}

\begin{document}

\maketitle

\begin{abstract}
We formalize AI alignment as a multi-objective optimization problem called $\langle M,N,\eps,\delta\rangle$-agreement, in which a set of $N$ agents (including humans) must reach approximate ($\eps$) agreement across $M$ candidate objectives, with probability at least $1-\delta$.
Analyzing communication complexity, we prove an information-theoretic lower bound showing that once either $M$ or $N$ is large enough, no amount of computational power or rationality can avoid intrinsic alignment overheads.
This establishes rigorous limits to alignment \emph{itself}, not merely to particular methods, clarifying a ``No-Free-Lunch'' principle: encoding ``all human values'' is inherently intractable and must be managed through consensus-driven reduction or prioritization of objectives.
Complementing this impossibility result, we construct explicit algorithms as achievability certificates for alignment under both unbounded and bounded rationality with noisy communication.
Even in these best-case regimes, our bounded-agent and sampling analysis shows that with large task spaces ($D$) and finite samples, \emph{reward hacking is globally inevitable}: rare high-loss states are systematically under-covered, implying scalable oversight must target safety-critical slices rather than uniform coverage.
Together, these results identify fundamental complexity barriers---tasks ($M$), agents ($N$), and state-space size ($D$)---and offer principles for more scalable human-AI collaboration.
\end{abstract}


\begin{figure*}[htbp]
    \centering
\includegraphics[width=0.8\linewidth]{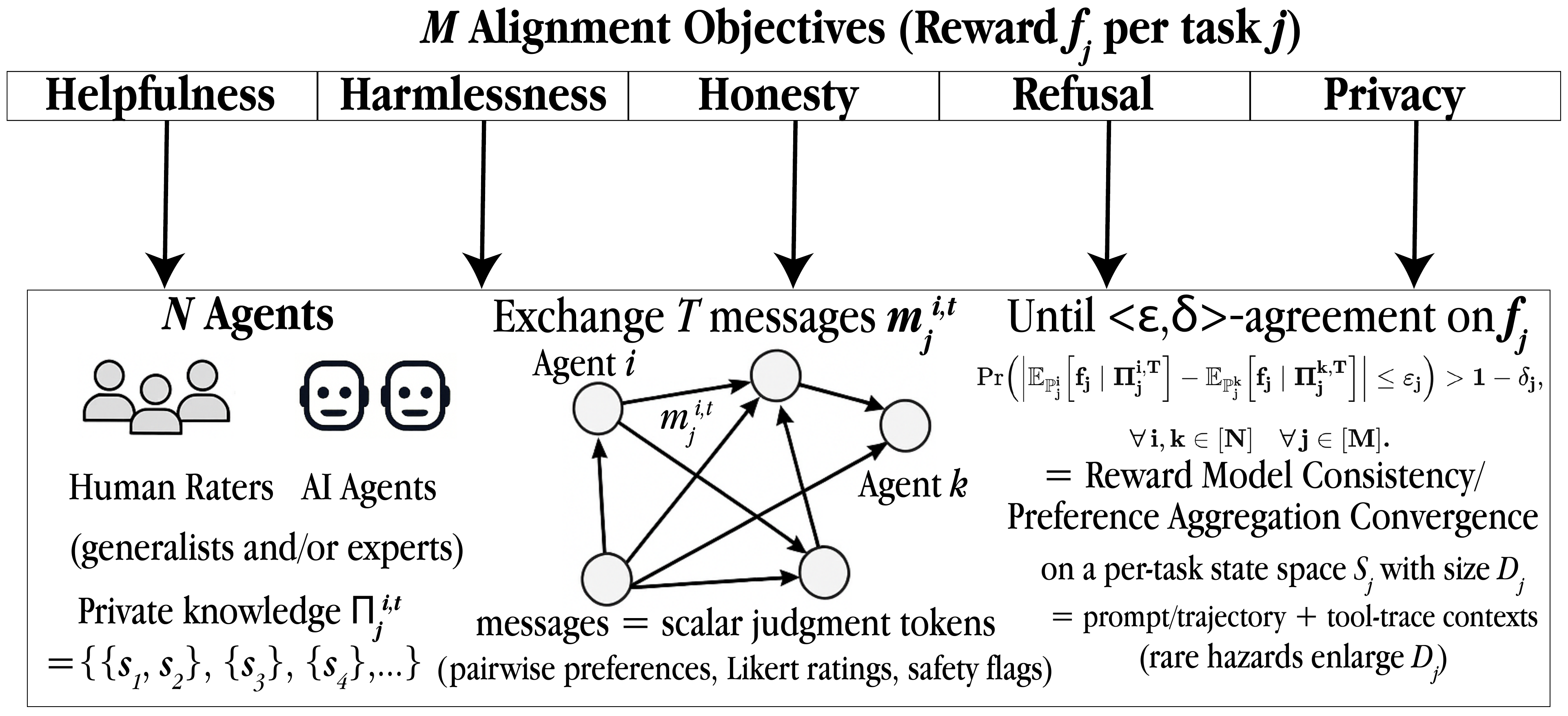}
    \caption{Mapping our $\agree$-agreement to current RLHF/DPO/Constitutional AI pipelines.}
    \label{fig:mapping}
\end{figure*}

\begin{table*}[ht]
\centering
\footnotesize
\setlength{\tabcolsep}{2pt}
\begin{tabular}{lcccccccccc}
\toprule
\textbf{Framework} &
\textbf{No‑CPA} & 
\textbf{Approx} & 
\textbf{Multi‑$M$} & 
\textbf{Multi‑$N$} & 
\textbf{Hist.} &
\textbf{Bnd.} &          
\textbf{Asym.} &        
\textbf{Noise} & 
\textbf{Upper} & 
\textbf{Lower} \\        
\midrule
Aumann (1976)                                              & \xmark & \xmark & \xmark & \xmark & \cmark & \xmark & \xmark & \xmark & \xmark & \xmark \\
Aaronson $\langle\eps,\delta\rangle$ (2005)         & \xmark & \cmark & \xmark & \cmark & \cmark & \cmark & \xmark & \cmark & \cmark & \cmark \\
Almost CP (Hellman and Samet 2012; Hellman 2013)                                   & \cmark & \xmark & \xmark & \cmark & \cmark & \xmark & \xmark & \xmark & \xmark & \xmark \\
CIRL (Hadfield‑Menell et al.\ 2016)                        & \xmark & \cmark & \xmark & \xmark & \xmark & \cmark & \xmark & \cmark & \cmark & \xmark \\
Iterated Amplification (Christiano et al.\ 2018)           & \cmark & \cmark & \xmark & \xmark & \cmark & \cmark & \xmark & \cmark & \cmark & \xmark \\
Debate (Irving et el. 2018; Cohen et al.\ 2023, 2025)     & \cmark & \xmark & \xmark & \xmark & \cmark & \cmark & \xmark & \cmark & \cmark & \xmark \\
Tractable Agreement (Collina et al.\ 2025)                 & \cmark & \cmark & \xmark & \cmark & \cmark & \cmark & \xmark & \xmark & \cmark & \xmark \\
\midrule
\textbf{\agree‑agreement (Ours)}                           &
\cmark & \cmark & \cmark & \cmark & \cmark & \cmark & \cmark & \cmark & \cmark & \cmark \\
\bottomrule
\end{tabular}
\caption{Positive capabilities (\cmark) across frameworks.  
\textbf{No‑CPA}: no common‑prior assumption (CPA);  
\textbf{Approx}: allows $\eps$‑approximate agreement;  
\textbf{Multi‑$M$} / \textbf{Multi‑$N$}: supports multiple tasks / many agents;  
\textbf{Hist.}: handles rich (non‑Markovian) histories;  
\textbf{Bnd.}: works for computationally \emph{bounded} agents; 
\textbf{Asym.}: tolerates \emph{asymmetric} evaluation or interaction costs;  
\textbf{Noise}: robust to noisy messages or judgments;  
\textbf{Upper}: provides explicit upper bounds (algorithms)---these can be useful as achievability certificates rather than prescriptions;  
\textbf{Lower}: proves lower bounds.  
Our $\langle M,N,\eps,\delta\rangle$‑agreement satisfies every criterion.}
\label{tab:related_work}
\end{table*}
\nocite{christiano2018supervising}

\section{Introduction}
\label{sec:introduction}
Rapid progress in artificial intelligence (AI) technologies, increasingly deployed across critical economic and societal domains, underscores the importance of ensuring these systems align with human intentions and values—a challenge known as the \emph{value alignment problem}~\citep{russell2015research,amodei2016concrete,soares2018value}.
Current alignment research frequently addresses immediate practical concerns, such as preventing jailbreaks in large language models~\citep{ji2023ai,guan2024deliberative,hubinger2024sleeper}.
While essential, these approaches largely focus on specific AI architectures and lack general, theoretically proven guarantees for alignment as systems approach human-level general capability.

Existing theoretical frameworks, notably AI Safety via Debate~\citep{irving2018ai,brown2023scalable,brown2025avoiding} and Cooperative Inverse Reinforcement Learning (CIRL)~\citep{hadfield2016cooperative}, have significantly advanced our understanding by providing formal guarantees of alignment in specific scenarios.
Debate effectively leverages interactive proofs to isolate misalignment through zero-sum debate games, though it relies critically on exact verification by a correct and unbiased human judge and computational tractability constraints.
CIRL successfully formulates alignment as a cooperative partial-information game reducible to a POMDP, allowing an elegant characterization of optimal joint policies under shared uncertainty~\citep{hadfield2016cooperative}.
However, CIRL implicitly assumes common priors and employs a Markovian assumption, potentially limiting agents' ability to leverage richer historical contexts for alignment.
While these methods represent important theoretical progress, their simplifying assumptions restrict broader applicability and leave open questions about alignment scenarios involving diverse knowledge states, richer agent interactions, or more complex objectives.
This underscores a crucial theoretical gap: no unified framework currently addresses alignment under minimal assumptions while rigorously identifying \emph{intrinsic} barriers independent of specific modeling choices. 
We propose that prior alignment approaches implicitly rely on underlying conceptual foundations involving iterative reasoning, mutual updating, common knowledge, and convergence under shared frameworks. 

To bridge this gap, we explicitly formalize these elements within an assumption-light framework called \agree-agreement (\S\ref{ss:results-framework}), which models alignment as a multi-objective optimization problem involving minimally capable agents and allows us to rigorously analyze alignment in highly general contexts.
In \agree-agreement, a group of agents (including humans) must achieve approximate consensus across multiple objectives with high probability.
We show in Table~\ref{tab:related_work} that our framework generalizes previous alignment approaches by relaxing their strong assumptions, thus enabling analysis under a broad set of conditions.

We then rigorously establish intrinsic, method-independent complexity-theoretic barriers to alignment, formalizing a fundamental ``No-Free-Lunch'' principle in Proposition~\ref{prop:lb}: attempting to encode all human values inevitably incurs alignment overheads, regardless of agent computational power or rationality.
Complementing this impossibility result, we also provide explicit algorithms in \S\ref{sec:ub}, not as prescriptions, but as \emph{achievability certificates} for both computationally unbounded and bounded rational agents, alongside closely matching lower bounds in \S\ref{sec:lb}.
Taken together, our results yield guidelines (\S\ref{sec:discussion}) clarifying the overall landscape of alignment and providing practical pathways for more scalable human-AI collaboration.

\section{Related Work}
\label{sec:related_work}
We summarize the key assumptions and features of previous alignment and agreement frameworks in Table~\ref{tab:related_work}, positioning our \agree‑agreement framework within the literature. 
Where earlier methods typically require common priors, exact agreement, single‑objective settings, or Markovian dynamics, our framework drops these assumptions, scales to many tasks and agents, tolerates noisy non‑Markovian exchanges, and supports bounded, cost‑asymmetric participants. 
Because it operates at the scalar‑reward level that dominates real‑world AI‑safety work, it can absorb \emph{any} previous protocol---including two‑agent, no‑common‑prior schemes such as \citet{collina2025tractable}---and lift them to our more general multi‑task, multi‑agent, asymmetric, noisy setting, obtaining universal lower bounds and closely matching upper bounds for broad classes of natural protocols.

\section{\agree-Agreement Framework}
\label{ss:results-framework}
\paragraph{Setup.}
For $M$ tasks $\M:=\{1,\dots,M\}$ and $N$ agents (humans and AIs) $\N:=\{1,\dots,N\}$, each task $j\in\M$ has finite state-space $S_j$ with $|S_j|=D_j$, bounded\footnote{Since $S_j$ is finite, note that any $f_j: S_j\to\R$ can be rescaled to $[0,1]$.} objective $f_j:S_j\to[0,1]$, and probability simplex $\Delta(S_j)\subseteq\R^{D_j}$.

Each agent $i\in\N$ begins with an \emph{individual} prior $\prob^{\,i}_j$ over $S_j$; we \textbf{do not} assume a common prior (CPA).
The \emph{prior distance}, as introduced by~\citet{hellman2013almost}, is:
\begin{equation}
\label{eq:prior-dist}
  \nu_j
  \;=\;
  \min_{\tuple{x_1,x_2}\in\prob^{i}_j\times\prob^{k}_j,\;
        p\in\Delta(S_j)}
  \|x_1-p\|_1 + \|x_2 - p\|_1
\end{equation}
measures disagreement between any pair of priors; $\nu_j=0$
iff a true common prior exists.
Note by the triangle inequality, $\nu_j \ge \|\prob^{i}_j-\prob^{k}_j\|_1$ if these priors are \emph{sets} of prior distributions per agent, and holds with equality for the typical setting of single prior distributions per agent.
This notion of prior distance captures the smallest change in beliefs needed for agents to share a common prior---a measure of how close the agents already are to agreement.

\paragraph{Information flow.}
At round $t\ge 0$ each agent $i$ holds a
\emph{knowledge partition} $\Pi^{i,t}_j=\left\{C^{\,i,t}_{j,k}\right\}_k$ of $S_j$.
Each cell $C^{\,i,t}_{j,k} \subseteq S_j$ is a set of states, so $\Pi^{i,t}_j(s_j)$ is the set of states in $S_j$ that agent $i$ finds possible at time $t$, given that the true state of the world is $s_j \in S_j$.
Agents broadcast real-valued messages $m^{i,t}_j \in [0,1]$  and \emph{refine} their partitions: $\Pi^{i,t+1}_j \subseteq \Pi^{i,t}_j$, and update their posterior belief distributions $\tau^{i,t}_j$, giving rise to the \emph{type profile} across the $N$ agents $\tau^t_j = (\tau^{1,t}_j,\dots\tau^{N,t}_j)$.
All knowledge partitions are common knowledge, ensuring the standard~\citet{aumann1976agreeing,aumann1999interactive} update semantics, but without assuming CPA.
Please see Appendix~\S\ref{app:agreement-protocol} for full details.

\paragraph{\agree‑Agreement Criterion.}
Fix tolerances $\eps_j,\delta_j \in(0,1)$.
After $T$ rounds, the agents \agree\textbf{-agree} if:
\begin{equation}
\label{eq:eps-delta}
\begin{split}
  &\Pr\!\left(
    \bigl|\,
      \Eji{f_j \mid \Pi^{\,i,T}_j}
      -\Ejk{f_j \mid \Pi^{\,k,T}_j}
    \bigr|
    \le \eps_j
  \right)
  > 1-\delta_j,\\
  &\quad\forall i,k\in\N \quad\forall j\in\M.
\end{split}
\end{equation}
Exact agreement is the special case $\eps_j=\delta_j=0$.

\paragraph{Why this ``best‑case'' model matters.}
\agree{} subsumes classical exact (and inexact) agreement
results and all prior alignment formalisms that we examine in
Table~\ref{tab:related_work} (visualized in Figure~\ref{fig:mapping}).
If alignment is \emph{hard even here}---with fully rational and \emph{computationally unbounded} agents, ideal message delivery, and no exogenous adversary---then practical settings with bounded rationality, noisy channels, or strategic misreporting can only be harder (therefore we want to avoid them).  
\S\ref{ss:results-bounded} quantifies exactly how
much harder.

\paragraph{Notation.}
Let $D:=\max_{j\in\M} D_j$ when used in the context of upper bounds (and $D:=\min_{j\in\M} D_j$ for lower bounds), let $\eps:=\min_{j\in\M}\eps_j$, and write $\mathbb{P},\mathbb{E}$ for probability and expectation.
The power‑set function is $\powset{\cdot}$.
All omitted proofs and implementation details---e.g. explicit message formats, the spanning‑tree protocol for refinement, and the LP procedure \textsc{ConstructCommonPrior} used in \S\ref{ss:results-bounded}---are provided in the Appendix.

\section{Lower Bounds}
\label{sec:lb}
Below is the best lower bound we can prove for \agree-agreement, across \emph{all} possible communication protocols:
\begin{proposition}[General Lower Bound]\label{prop:lb}
There exist functions $f_j$, input sets $S_j$, and prior distributions $\{\prob_j^i\}^{i \in \N}$ for all $j \in \M$, such that any protocol among $N$ agents needs to exchange $\Omega\left(M\,N^2\,\log\left(1/\eps\right)\right)$ bits\footnote{Note, unlike our upper bounds in Theorem~\ref{thm:ub} and Proposition~\ref{prop:disc}, we use bits in the lower bound in order to apply to \emph{all} possible protocols (continuous or discrete), regardless of how many bits are encoded per message. The upper bounds have to use messages (rounds) to describe either a continuous protocol (potentially infinitely many bits) as in Theorem~\ref{thm:ub}, or a discrete protocol as in Proposition~\ref{prop:disc}.} to achieve $\agree$-agreement on $\{f_j\}_{j \in \M}$, for $\eps$ bounded below by $\min_{j\in\M} \eps_j$.
\end{proposition}

Thus, by Proposition~\ref{prop:lb}, there does \emph{not} exist any \agree-agreement protocol (deterministic or randomized) that can exchange less than $\Omega\left(M\,N^2\,\log\left(1/\eps\right)\right)$ bits for \emph{all} $f_j, S_j$, and prior distributions $\{\prob_j^i\}^{i \in \N}$.
For if it did, then we would reach a contradiction for the particular construction in Proposition~\ref{prop:lb}.
Note that the linear dependence on $M$ can mean an exponential number of bits in the lower bound if we have that many \emph{distinct} tasks (or agents), e.g. if $M = \Theta(D)$ for a large task state space $D$.

By considering the natural subclass of \emph{smooth protocols}---where agents' posterior beliefs at \agree-agreement time must not diverge more than their initial priors, measured in total variation distance---we obtain a strictly improved lower bound:
\begin{proposition}[``Smooth'' Protocol Lower Bound]\label{prop:lb:2}
Let the number of tasks $M \ge 2$, and for each task $j \in \M$, let the task state space size $D_j > 2$, $\eps \le \eps_j$, $\delta_j < \nu/2$, and $0 < \nu \le 1$.
Furthermore, assume the protocol is smooth in that the total variation distance of the posteriors of the agents once \agree-agreement is reached is $\le c\nu$ for $c < \tfrac{1}{2} - \tfrac{\delta_j}{\nu}$. 
There exist functions $f_j$, input sets $S_j$, and prior distributions $\{\prob_j^i\}^{i \in \N}$ with prior distance $\nu_j \ge \nu$, such that any smooth protocol among $N$ agents needs to exchange:
\begin{equation*}
\Omega\left(M\,N^2\,\left(\nu + \log\left(1/\eps\right)\right)\right)
\end{equation*}
bits to achieve $\agree$-agreement on $\{f_j\}_{j \in \M}$.
\end{proposition}
Both lower bounds in Propositions~\ref{prop:lb} and~\ref{prop:lb:2} demonstrate that gaining consensus on a small list of $M$ values that we want AI systems to have, will be essential for scalable alignment.

Finally, we consider a related smoothness condition—namely, the broad class of \emph{bounded-Bayes-factor (BBF)} protocols---in which each message bit alters message likelihoods by at most a constant multiplicative factor. 
This assumption naturally captures realistic message-passing behavior, since rational and bounded agents typically update their beliefs incrementally rather than abruptly shifting posterior distributions after receiving a single message. 
Under this mild condition, we examine a natural setting: agents initially separated by prior distance $\nu$ first establish a \emph{common prior} by satisfying the canonical equalities of \citet{hellman2012common} (displayed in Algorithm~\ref{alg:construct} in Appendix~\S\ref{app:runtime}), and subsequently condition on this shared prior to achieve \agree-agreement. 
They showed that for tight and connected knowledge partitions (defined below), these canonical equalities are automatically preserved under standard Bayesian updating; hence, our construction needs no further behavioral constraints beyond standard Bayesian rationality.

Under these reasonable conditions, our lower bound strengthens to include an extra multiplicative factor of $D:=\min_j D_j$, the smallest state-space size across the $M$ tasks. 
Thus, this refined lower bound more closely matches the general upper-bound results from \S\ref{sec:ub} (cf. Algorithm~\ref{alg:agree}) for this protocol class within an additive polynomial term in $M,N,\eps$, and $\delta$.
\begin{proposition}[Canonical-Equality BBF Protocol Lower Bound]
\label{prop:lb:3}
Let $M\ge 2$ be the number of tasks and let each task $j$ have a finite state‑space $S_j$ with size $D_j>2$.
For every $j$, let the \emph{initial} knowledge profiles of the $N$ agents, $(\Pi_j^{1,0},\dots,\Pi_j^{N,0})$, be
\begin{enumerate}
\item \emph{connected:} the alternation graph on states is connected, i.e. $\bigwedge_i\Pi_j^{i,0}=\{S_j\}$, so every two states are linked by an alternating chain of states; and
\item \emph{tight:} that graph becomes disconnected if any edge is removed (unique chain property).
\end{enumerate}

Assume the message‑passing protocol is {BBF$(\beta)$} for some $\beta > 1$: every $b$‑bit message $m_j^{i,t}$ satisfies $\beta^{-b}\!\le\!\Pr[m_{j}^{i,t}\mid s_j, \Pi^{i,t-1}_j(s_j)]/\Pr[m_{j}^{i,t}\mid s'_j, \Pi^{i,t-1}_j(s'_j)]\!\le\!\beta^{\,b}$.
Then there exist payoff functions $f_j:S_j\!\to\![0,1]$ and priors $\{\mathbb P_j^{\,i}\}_{i\in[N]}$ with pairwise distance $\nu_j\ge\nu$, $0<\nu\le 1$, such that any BBF$(\beta)$ protocol attaining \agree‑agreement via the canonical equalities of~\citet{hellman2012common} must exchange at least
\begin{equation*}
   \Omega\left(
        M\,N^{2}\,[\,D\nu+\log(1/\eps)\,]
     \right),
   \qquad
   D:=\min_{j\in\M}D_j ,
\end{equation*}
bits in the worst case (implicit constant $=1/\log\beta$), where the accuracy parameter $0<\eps\le\eps_j<1$.
\end{proposition}

\section{Convergence of \agree-agreement}
\label{sec:ub}
Given these lower bounds, a natural question is whether \agree-agreement is achievable at all---especially since the agents begin without a common prior. 
In this section, we demonstrate that it is indeed achievable, providing explicit algorithms and upper bounds on convergence not only for idealized, unbounded agents but also under realistic constraints such as message discretization and computational boundedness.
Here we prove the general upper bound:
\begin{theorem}\label{thm:ub}
$N$ rational agents will \agree-agree with overall failure probability $\delta$ across $M$ tasks, as defined in \eqref{eq:eps-delta}, after $T = O\left(MN^2 D + \dfrac{M^3N^7}{\eps^2\delta^2}\right)$ messages, where $D := \max_{j \in \M} D_j$ and $\eps := \min_{j \in \M}\eps_j$.
\end{theorem}
For an explicit algorithm, see Algorithm~\ref{alg:agree}---we detail the reasoning behind this algorithm below.

\begin{algorithm}[ht]
\caption{\textsc{\agree}-Agreement}\label{alg:agree}
\begin{algorithmic}[1]          
\REQUIRE $N$ agents with initial partitions
         $\{\Pi^{i,0}_j\}_{i=1}^N$ for each task $j\in\M$;
         protocol $\mathcal{P}$; \textsc{ConstructCommonPrior} defined in Algorithm~\ref{alg:construct};
         $\tuple{\eps,\delta}$-agreement protocol $\mathcal{A}$
\ENSURE  Agents reach $\tuple{\eps_j,\delta_j}$-agreement for all $M$ tasks

\FOR{$j \gets 1$ to $M$}
  \STATE $t \gets 0$
  \REPEAT
    \STATE $t \gets t+1$
    \FORALL{agent $i \in \N$}
      \STATE send $m^{i,t}_j$ via $\mathcal{P}$
      \STATE $\Pi^{i,t}_j \gets 
             \textsc{RefinePartition}\bigl(\Pi^{i,t-1}_j, m^{\cdot,t}_j\bigr)$
    \ENDFOR
    \STATE $\CP_j \gets
            \textsc{ConstructCommonPrior}\!\bigl(
            \{\Pi^{i,t}_j\}_{i=1}^{N},$
    \STATEx $\{\tau^{i,t}_j\}_{i=1}^{N}\bigr)$
  \UNTIL{$\CP_j \neq$ \textsc{Infeasible}}
  \STATE Condition all agents on $\CP_j$
  \STATE $\textsc{RunCPAgreement}\left(\mathcal{A},\mathcal{P},\CP_j,f_j,\eps_j,\delta_j\right)$
\ENDFOR
\end{algorithmic}
\end{algorithm}

First, we need to figure out at most how many messages need to be exchanged to guarantee at least one proper refinement.
To do so, we will have the $N$ agents communicate using the ``spanning-tree'' protocol of~\citet[\S 3.3]{aaronson2005complexity}, which we generalize to the multi-task, no common prior, setting below:

\begin{lemma}[Proper Refinement Message Mapping Lemma]\label{lem:spanning-tree-refinement}
If $N$ agents communicate via a spanning‐tree protocol for task $j$, where
$g_j \in \Nat$ is the diameter of the chosen spanning trees, then as long as they have not yet reached agreement, it takes $O(g_j)= O(N)$ messages before at least one agent's knowledge partition is properly refined.
\end{lemma}
\begin{proof}

Let $G_j$ be a strongly connected directed graph with vertices $v\in\N$ (one per agent), enabling communication of expectations $E^{i,t}_j$ along edges. 
(We need the strongly connected requirement on $G_j$, since otherwise the agents may not reach agreement for trivial reasons if they cannot reach one another.)
Without loss of generality, let ${SP}^1_j$ and ${SP}^2_j$ be minimum-diameter spanning trees of $G_j$, each rooted at agent~1, with ${SP}^1_j$ pointing outward from agent~1 and ${SP}^2_j$ inward toward agent~1, each of diameter at most~$g_j$.

Define orderings $\mathcal{O}^1_j$ (resp.\ $\mathcal{O}^2_j$) of edges in ${SP}^1_j$ (resp.\ ${SP}^2_j$) so each edge $(i\to k)$ appears only after edges $(\ell\to i)$, except when $i$ is the root (or leaf, in inward trees). Construct $\mathrm{AgentOrdering}_j$ by cycling through $\mathcal{O}^1_j,\mathcal{O}^2_j,\dots$, where in each round~$t$ the tail agent of $\mathrm{AgentOrdering}_j(t)$ sends its current expectation. Thus, every block of $O(g_j)$ transmissions forwards each agent's updated message along both trees, reaching all others.

Consequently, disagreement between any agents $i$ and $k$ leads to at least one agent receiving a ``surprising'' message within these $O(g_j)$ transmissions (worst-case occurs when $i,k$ are on opposite ends of $G_j$), causing a partition refinement. Thus, without agreement, at least one refinement occurs every $O(g_j)$ messages.

Note $g_j=O(N)$ if $G_j$ is a worst-case ring topology; more favorable topologies yield $g_j\ll N$, but we assume worst-case generality to subsume any specific cases.
\end{proof}

Next, we prove an important (for our purposes) lemma, which is an extension of~\citet[Theorem 2]{hellman2013almost}'s result on almost common priors to our $M$-function message setting:
\begin{lemma}[Common Prior Lemma]\label{lem:common-prior}
If $N$ agents have prior distance $\nu_j$, as defined in \eqref{eq:prior-dist}, for a task $j\in \M$ with task state space $S_j$, then after $O\left(N^2 D_j \right)$ messages, they will have a common prior $\CP_j$ with probability 1 over their type profiles.
\end{lemma}

Once the agents reach a common prior $\CP_j$, they can then condition on that for the rest of their conversation to reach the desired $1-\delta_j$ $\eps_j$-agreement threshold (cf. Step 12 of Algorithm~\ref{alg:agree}).
We assume this is $O(1)$ to compute for now as the agents are computationally unbounded, but we will remove this assumption in \S\ref{ss:results-bounded}, and instead use Algorithm~\ref{alg:construct} (Appendix~\S\ref{app:runtime}) for an efficient explicit construction via LP feasibility of posterior belief ratios.

Therefore, for each task $j$, we have reduced the problem now to Aaronson's $\tuple{\eps, \delta}$-agreement framework~\citep{aaronson2005complexity}, and as he shows, the subsequent steps conditioning on a common prior become unbiased random walks with step size roughly $\eps_j$.
With some slight modifications, this allows us to give a worst-case bound on the number of remaining steps in our \agree-agreement setting:
\begin{lemma}\label{lem:spanning-tree}
For all $f_j$ and $\CP_j$, the $N$ agents will globally $\tuple{\eps_j, \delta_j}$-agree after $O\left({N^7}/{\left(\delta_j\eps_j\right)^2}\right)$ additional messages.
\end{lemma}
\begin{proof}
By~\citet[Theorem 10]{aaronson2005complexity}, the $N$ agents will \emph{pairwise} $\tuple{\eps_j, \delta_j}$-agree after $O\left(\left(Ng_j^2\right)/{\left(\delta_j\eps_j\right)^2}\right)$ messages when they condition on $\CP_j$, where $g_j$ is the diameter of the spanning-tree protocol they use.
Furthermore, we will need to have them $\tuple{\eps_j, \delta_j/N^2}$-agree pairwise so that they \emph{globally} $\tuple{\eps_j, \delta_j}$-agree.
Taking $g_j = O(N)$ for the worst-case ring topology gives us the above bound.
\end{proof}
By Lemmas~\ref{lem:common-prior} and~\ref{lem:spanning-tree}, for \emph{each} $j \in \M$, we need $O\left(N^2 D_j + \dfrac{N^7}{\left(\delta_j\eps_j\right)^2}\right)$ messages for the $N$ agents to reach $\tuple{M=1, N, \eps_j, \delta_j}$-agreement.
Next, select a uniform $\delta$ such that $\delta_j \le \delta/M$, for all $j \in \M$.
Therefore, by a union bound, we get the full upper bound in Theorem~\ref{thm:ub} with total probability $\ge 1- \delta$, across \emph{all} $M$ tasks, by \emph{maximizing} the bound above by taking $D := \max_{j \in \M} D_j$ and $\eps := \min_{j \in \M}\eps_j$, and scaling by $M$.

\subsection{Discretized Extension}
\label{ss:results-disc}
A natural extension of Theorem~\ref{thm:ub} is if the agents do not communicate their full real-valued expectation (which may require infinitely many bits), but a discretized version of the current expectation, corresponding to whether it is above or below a given threshold (defined below), e.g. ``High'', ``Medium'', or ``Low'' (requiring only 2 bits).
We prove convergence in this case, and show that the bound from Theorem~\ref{thm:ub} remains unchanged in this setting.
Discretization is important to show convergence and complexity analysis for, since this most closely matches real-world constraints (e.g. LLM agents use discrete, real-valued tokens), as opposed to infinite-bit real valued messages.
\begin{proposition}[Discretized Extension]\label{prop:disc}
If $N$ agents only communicate their \emph{discretized} expectations, then they will \agree-agree with overall failure probability $\delta$ across $M$ tasks as defined in \eqref{eq:eps-delta}, after $T = O\left(MN^2 D + \dfrac{M^3N^7}{\eps^2\delta^2}\right)$ messages, where $D := \max_{j \in \M} D_j$ and $\eps := \min_{j \in \M}\eps_j$.
\end{proposition}
Our discretized three‑bucket protocol itself is general and imposes no BBF constraint---in Appendix \S\ref{sec:propdisc-proof} we show it can be made BBF$(3)$‑compliant with small overhead.
Thus, by the lower bound from Proposition~\ref{prop:lb:3}, for the broad and natural class of canonical-equality BBF protocols, our upper bound in Proposition~\ref{prop:disc} is tight up to an additive polynomial term after converting from messages to bits.

\subsection{Computationally Bounded Agents}
\label{ss:results-bounded}
Thus far, we analyzed computationally unbounded agents, implicitly assuming $O(1)$ time for constructing and sending messages, finding common priors, and sampling distributions. Even under these idealized conditions, the linear scaling in Theorem~\ref{thm:ub} becomes significant if the task space $D$ or number of tasks $M$ is exponentially large.

However, realistic agents, such as current LLMs, are computationally bounded, and message passing may be noisy, e.g., due to obfuscated intent~\citep{barnes2020debateobf}. 
Thus, we now analyze the complexity of $N$ computationally \emph{bounded} rational agents. 
Moreover, since querying humans typically costs more than querying AI agents, we differentiate between $q$ humans (each taking $T_H$ time steps) and $N-q$ AI agents (each taking $T_{AI}$ time steps), encompassing recent multi-step reasoning models~\citep{openai_o1, gemini_2_0_flash}. 
Without loss of generality, we assume uniform times within these two groups and analyze complexity based on two basic subroutines:
\begin{requirement}[Basic Capabilities of Bounded Agents]\label{req:bounded-cap}
We expect the agents to be able to:
\begin{enumerate}
\item \textbf{Evaluation:} The $N$ agents can each evaluate $f_j(s_j)$ for any state $s_j \in S_j$, taking time $T_{\text{eval},a}$ steps for $a \in \{H, AI\}$.
\item \textbf{Sampling:} The $N$ agents can sample from the \emph{unconditional} distribution of any other agent, such as their prior $\mathbb{P}_j^i$, taking time $T_{\mathrm{sample},a}$ steps for $a \in \{H, AI\}$.
\end{enumerate}
\end{requirement}
We treat these subroutines as black boxes: agents lack explicit descriptions of $f_j$ and distributions, learning about them solely through these operations. Analogous to CIRL~\citep{hadfield2016cooperative}, this setup captures realistic alignment scenarios where the correctness of a task outcome can be verified without specifying each intermediate step. 
Consequently, our complexity results are broadly applicable, expressed in terms of $T_{\text{eval},H}$, $T_{\text{eval},AI}$, $T_{\mathrm{sample},H}$, and $T_{\mathrm{sample},AI}$.

These minimal subroutines enable agents to estimate each other's expectations, an essential capability for alignment. Importantly, exact computation is unnecessary; probabilistic evaluation in polynomial time suffices (as will become clear in the proof of Theorem~\ref{thm:bounded}, due to the exponential blow-up). 
The sampling subroutine further serves as a bounded version of the standard assumption that agents know each other's knowledge partitions through shared states~\citep{aumann1976agreeing,aumann1999interactive}. 
This corresponds to agents possessing a bounded ``theory of mind''~\citep{ho2022planning} about one another.

Finally, as we can no longer assume $O(1)$ time complexity for constructing a common prior (unlike in the unbounded agent setting), we introduce an explicit randomized polynomial-time algorithm for doing so with high probability, Algorithm~\ref{alg:construct}.
We refer the reader to Appendix \S\ref{app:runtime} for proofs related to Algorithm~\ref{alg:construct}.
Specifically, Lemma~\ref{lem:cp-correctness} (correctness), Lemma~\ref{lem:cp-runtime} (runtime analysis), and Lemma~\ref{lem:approx-cp} (inexact posterior access setting).

In what follows, define
\begin{equation*}
\begin{aligned}
T_{N,q} :=\;&
  q\,T_{\mathrm{sample},H} + (N-q)\,T_{\mathrm{sample},AI} \\
           &+\,q\,T_{\text{eval},H} + (N-q)\,T_{\text{eval},AI}.
\end{aligned}
\end{equation*}
The above considerations lead to the following theorem in the bounded agent setting:

\begin{theorem}[Bounded Agents Eventually Agree]\label{thm:bounded}
Let there be $N$ computationally bounded rational agents (consisting of $1 \le q < N$ humans and $N-q \ge 1$ AI agents), with the capabilities in Requirement~\ref{req:bounded-cap}.
The agents pass messages according to the sampling tree protocol (detailed in Appendix \S\ref{app:sampling-tree}) with branching factor of $B \ge 1/\alpha$, and added triangular noise of width $\le 2\alpha$, where $\eps/50 \le \alpha \le \eps/40$.
Let $\delta^{\text{find\_CP}}$ be the maximal failure probability of the agents to find a task-specific common prior across all $M$ tasks, and let $\delta^{\text{agree\_CP}}$ be the maximal failure probability of the agents to come to \agree-agreement across all $M$ tasks once they condition on a common prior, where $\delta^{\text{find\_CP}} + \delta^{\text{agree\_CP}} < \delta$.
For the $N$ computationally bounded agents to \agree-agree with total probability $\ge 1-\delta$, takes time
\begin{equation*}
O\left(
  M\,T_{N,q}\!
  \left(
B^{\,N^{2}D\frac{\ln\!\bigl(\delta^{\text{find\_CP}}/(3MN^{2}D)\bigr)}{\ln(1/\alpha)}}
      \;+\;
      B^{\,\frac{9M^{2}N^{7}}{(\delta^{\text{agree\_CP}}\eps)^{2}}}
  \right)
\right).
\end{equation*}

In other words, just in the first term alone, \emph{exponential} in the task space size $D$ and number of agents $N$ (and exponential in the number of tasks $M$ in the second term).
So if the task space size is in turn exponential in the input size, then this would already be \emph{\underline{doubly exponential}} in the input size!
\end{theorem}

We now clarify why we let $B$ be a parameter, and give a concrete example of how bad this exponential dependence can be.
Intuitively, we can think of $B$ as a ``gauge'' on how distinguishable the bounded agents are from ``true'' unbounded Bayesians, and will allow us to give an explicit desired value for $B$.
Recognizing the issue of computational boundedness of agents in the real world, \citet{hanson2003bayesian} introduced the notion of \emph{Bayesian wannabes}: agents who estimate expectations as if they had sufficient computational resources. He showed that disagreement among Bayesian wannabes stems from computational limitations rather than differing information. 
Extending this idea, \citet{aaronson2005complexity} proposed a protocol ensuring that bounded agents appear statistically indistinguishable from true Bayesians to an external referee—effectively a ``Bayesian Turing Test''~\citep{turing1950computing} for rationality. 
Thus, $B$ explicitly quantifies this notion of bounded Bayesian indistinguishability.

We consider the $M$-function, $N$-agent generalization of this requirement (and \emph{without} common priors (CPA)), which we call a ``\emph{total Bayesian wannabe}'':
\begin{definition}[Total Bayesian Wannabe]\label{def:total-wannabe}
Let the $N$ agents have the capabilities in Requirement~\ref{req:bounded-cap}.
For each task $j \in \M$, let the transcript of $T$ messages exchanged between $N$ agents be denoted as $\Gamma_j := \tuple{m_j^1,\dots,m_j^T}$.
Let their initial, task-specific priors be denoted by $\{\prob_j^i\}^{i \in \N}$.
Let $\mathcal{B}(s_j)$ be the distribution over message transcripts if the $N$ agents are unbounded Bayesians, and the current task state is $s_j \in S_j$.
Analogously, let $\mathcal{W}(s_j)$ be the distribution over message transcripts if the $N$ agents are ``total Bayesian wannabes'', and the current task state is $s_j \in S_j$.
Then we require for all Boolean functions\footnote{Without loss of generality, we assume that the current task state $s_j$ and message transcript $\Gamma_j$ are encoded as binary strings.} $\Phi(s_j,\Gamma_j)$,
\begin{equation*}
\left\lVert
\begin{aligned}
  \mathbb{P}_{\substack{
    \Gamma_j \in \mathcal{W}(s_j)\\
    s_j \in \mathscr{S}_j}}
  \!\bigl[\Phi(s_j,\Gamma_j)=1\bigr] \\[4pt]
  \;-\; \mathbb{P}_{\substack{
    \Gamma_j \in \mathcal{B}(s_j)\\
    s_j \in \mathscr{S}_j}}
  \!\bigl[\Phi(s_j,\Gamma_j)=1\bigr]
\end{aligned}
\right\rVert_{1}
\;\le\; \rho_j,
\qquad \forall j \in \M,
\end{equation*}
where $\mathscr{S}_j := \{\prob_j^{\,i}\}_{i\in\N}$.
We can set $\rho_j \in \R$ as arbitrarily small as preferred, and it will be convenient to only consider a single $\rho := \min_{j \in \M} \rho_j$ without loss of generality (corresponding to the most ``stringent'' task $j$).
\end{definition}
We will show in Appendix \S\ref{app:runtime} that matching this requirement amounts to picking a large enough value for $B$, giving rise to the following corollary to Theorem~\ref{thm:bounded}:
\begin{corollary}[Total Bayesian Wannabes Agree]\label{cor:wannabe-agree}
Let there be $N$ total Bayesian wannabes, according to Definition~\ref{def:total-wannabe} (e.g. consisting of $1 \le q < N$ humans and $N-q \ge 1$ AI agents).
Let the branching factor of the sampling tree protocol be the same as before, $B \ge 1/\alpha$, with added triangular noise of width $\le 2\alpha$, where $\eps/50 \le \alpha \le \eps/40$.
Let $\delta^{\text{find\_CP}}$ be the maximal failure probability of the agents to find a task-specific common prior across all $M$ tasks, and let $\delta^{\text{agree\_CP}}$ be the maximal failure probability of the agents to come to \agree-agreement across all $M$ tasks once they condition on a common prior, where $\delta^{\text{find\_CP}} + \delta^{\text{agree\_CP}} < \delta$.
For the $N$ ``total Bayesian wannabes'' to \agree-agree with total probability $\ge 1-\delta$, takes time
\begin{equation*}
\begin{split}
O\!\Bigl(
  M\,T_{N,q}\!
  \bigl(
&B^{\,N^{2}D\frac{\ln\!\bigl(\delta^{\text{find\_CP}}/(3MN^{2}D)\bigr)}{\ln(1/\alpha)}}\\
    &+ (11/\alpha)^{\frac{729M^{6}N^{21}}{(\delta^{\text{agree\_CP}}\eps)^{6}}}
        \,\rho^{-\frac{18M^{2}N^{7}}{(\delta^{\text{agree\_CP}}\eps)^{2}}}
  \bigr)
\Bigr).
\end{split}
\end{equation*}
In other words, exponential time in the task space $D$, and by \eqref{eq:total-bayesian}, and with a large base in the second term if the ``total Bayesian wannabe'' threshold $\rho$ is made small.

Sharing a common prior amounts to removing the first term, yielding upper bounds that are still exponential in $\eps$ and $\delta$.
\end{corollary}
The proofs of Theorem~\ref{thm:bounded} and Corollary~\ref{cor:wannabe-agree} are quite technical (spanning 7 pages), so we defer them to Appendix \S\ref{app:proofs} for clarity.
The primary takeaway here is that computational boundedness can result in a severely exponential time slowdown in the agreement time, and especially so if you want the bounded agents to be \emph{statistically indistinguishable} in their interactions with each other from true unbounded Bayesians.

For example, even for $N=2$ agents with a common prior and liberal agreement threshold of $\eps = \delta = 1/2$ and ``total Bayesian wannabe'' threshold of $\rho = 1/2$ on one task ($M=1$), then $\alpha \ge 1/100$, the number of \emph{subroutine calls} (not even total runtime) would be around:
\begin{equation*}
O\left(\frac{\left(1100\right)^{\frac{1528823808}{\left(1/4\right)^6}}}{\left(1/2\right)^{\frac{2304}{\left(1/4\right)^2}}}\right) \approx O\left(10^{{10}^{13.27979}}\right),
\end{equation*}
would already far exceed the estimated~\citep[pg. 19]{Munafo_Notable_Numbers} number of atoms in the universe ($\sim 4.8 \times 10^{79}$)!
This illustrates the power of the \emph{unbounded} Bayesians we considered earlier in \S\ref{sec:lb}, and why the lower bounds there are worth paying attention to in practice.

Finally, note that in general under a sampling tree protocol, this exponential blow-up in task state space size $D$ is unavoidable (e.g. for rare, potentially unsafe, events):
\begin{proposition}[Needle-in-a-Haystack Sampling Tree Lower Bound]
\label{prop:needle}
Let $T_{N,q,\mathrm{sample}}:=qT_{\mathrm{sample},H}+(N-q)T_{\mathrm{sample},AI}$.
For \emph{any} sampling-tree protocol, a single task and a single pair of agents can be instantiated so that the two agents' priors differ by prior distance $\ge \nu$, yet the protocol must pre‑compute at least $\Omega\left(\nu^{-1}\right)$ unconditional samples before the first online message.
Consequently, for a particular ``needle'' prior construction of $\nu = \Theta\left(e^{-D}\right)$, we get lower bounds that are exponential in the task state space size $D$, needing $\Omega\left(M\,T_{N,q,\mathrm{sample}}\,e^D\right)$ wall-clock time.
\end{proposition}

\section{Discussion}
\label{sec:discussion}
\paragraph{Why study a ``Bayesian best‑case'' at all?}
One may object that real AI systems---and certainly humans---are not perfectly Bayesian reasoners, nor do they interact through ideal, lossless channels.
That is precisely the point: our results constitute an \emph{ideal benchmark}, before we build and deploy capable agents.
If alignment is information‑ or communication‑theoretically hard even for computationally \emph{unbounded}, rational Bayesians exchanging noiseless messages, then relaxing rationality and unboundedness assumptions, adding noise, strategic behavior, or adversarial tampering can exacerbate the difficulty, as we showed in \S\ref{ss:results-bounded}.
Our takeaways for AI safety are:

\begin{enumerate}
\item \textbf{Too many alignment values drives alignment cost.}
Our matched lower and upper bounds (tight up to polynomial terms in $M,N,\eps,\delta$) demonstrate a ``No‑Free‑Lunch'' principle: encoding an exponentially large or high‑entropy set of human values forces at least exponential communication even for \emph{unbounded} agents.
As a result, progress on value alignment / preference modeling should prioritize objective compression, delegation, or progressive disclosure rather than attempting one‑shot, full‑coverage specification.

\item \textbf{Reward hacking is \emph{globally} inevitable.} Proposition~\ref{prop:needle} shows that in large state spaces and with bounded agents, reward hacking arises unavoidably from finite sampling.
By Proposition~\ref{prop:lb:3}, this even happens for unbounded agents in large state spaces who communicate finite bits and update their expectations smoothly.
Scalable oversight is therefore not about uniform alignment, but about focusing on the parts of the state space that matter most. 
The engineering task ahead of us then is the \emph{mechanism design} problem of benchmarks and interactive protocols that target these safety-critical slices---via adversarial sampling, objective compression, and per-slice $\tuple{\eps,\delta}$ budgets---to certify coverage where it counts.

\item \textbf{Robustness depends on bounded rationality, memory, and theory of mind.}
Introducing bounded agents or even mild triangular noise can exponentially increase costs when protocols cannot exploit additional structure or restrict the task state space~\citep{ball2025don}; yet these assumptions were necessary to prove any alignment guarantees at all. 
Robust alignment must account for imperfect agents and noisy or obfuscated channels---but as we show in \S\ref{ss:results-bounded}, real-world agents with these three properties can degrade \emph{gracefully} rather than catastrophically.
   
\item \textbf{Tight bounds inform governance thresholds.}
For broad and natural protocol classes, our lower bounds are closely matched (up to polynomial terms) by constructive algorithms, enabling principled risk thresholds.
\end{enumerate}

\paragraph{Limitations and future work.}
Our results justify \emph{cautious optimism}: alignment is tractable in principle, yet only when we restrain objectives and exploit task structure with care.
``No-Free-Lunch'' does not preclude lunch---it simply forces wise menu choices.
At least three directions stand out:

\begin{enumerate}
\item \textbf{Minimal value sets.}  
Our lower bounds imply that having too many objectives is the surest route to inefficiency.  
A key open question is \emph{which} small, consensus‐worthy utility families guarantee high-probability safety.
In concurrent follow-up work~\citep{nayebi2025coresafety}, we identify such a small value set for corrigibility as defined by~\citet{soares2015corrigibility}, which was open for a decade.

\item \textbf{Structure‑exploiting interaction protocols.}  
Design \emph{multi-turn} agent interaction protocols (beyond single-shot RLHF) and evaluation benchmarks that stress-test the portions of state space most relevant for safety during deployment.
This can also be done at the post-training stage, and can augment existing RLHF pipelines.
\item \textbf{Beyond expectations under noise.}
\textbf{(i)} Can agreement on \emph{specific} risk measures cut communication costs relative to full‑expectation alignment?
We note that agreement on full expectations is not always required; given a task-specific utility function $U_j$, our framework already covers agreement on optimal actions, $\arg\max_a \mathbb{E}[U_j(a)]$ by having $f_j$ be the optimal action indicator.
Our framework also models rare events (Appendix~\S\ref{app:tail-risk}).
\textbf{(ii)} We found bounded derivative in the noise model was crucial for convergence (e.g. uniform noise does not suffice).
Studying richer obfuscation (e.g. learned steganography) will be essential for informing other robust safety thresholds.
\end{enumerate}
\newpage
\section*{Acknowledgements}
We thank the Burroughs Wellcome Fund (CASI award) for financial support.
We also thank Scott Aaronson, Andreas Haupt, Richard Hu, J. Zico Kolter, and Max Simchowitz for helpful discussions on AI safety in the early stages of this work, and Nina Balcan for feedback on a draft of the manuscript.

\bibliography{aaai2026}

\appendix
\include{appendix}
\end{document}

%% file: appendix.tex
\section{Notational Preliminaries}
\label{ss:results-notation}
We will use asymptotic notation throughout that is standard in computer science, but may not be in other fields.
The asymptotic notation is defined as follows:
\begin{itemize}
    \item $F(n) = O(G(n))$: There exist positive constants $c_1 > 0$ and $c_2 > 0$ such that $F(n) \leq c_1 + c_2G(n)$, for all $n \geq 0$.
    \item $F(n) = \tilde{O}(G(n))$: There exist positive constants $c_1, c_2$, and $k > 0$ such that $F(n) \leq c_1 + c_2 G(n) \log^k n$, for all $n \geq 0$.
    \item $F(n) = \Omega(G(n))$: Similarly, there exist positive constants $c_1$ and $c_2$ such that $F(n) \geq c_1 + c_2G(n)$, for all $n \geq 0$.
    \item $F(n) = \Theta(G(n))$: This indicates that $F(n) = O(G(n))$ and $F(n) = \Omega(G(n))$.
    In other words, $G(n)$ is a tight bound for $F(n)$.
\end{itemize}
For notational convenience, let
\begin{equation*}
E^{i,t}_j(s_j) := \Eji{f_j\mid \Pi_j^{i,t}(s_j)},
\end{equation*}
which is the expectation of $f_j$ of agent $i$ at timestep $t$, conditioned on its knowledge partition by then, starting from its \emph{own} prior ${\prob}^{i}_j$.
To simplify notation, we drop the argument $s_j \in S_j$.

\section{\agree-Agreement Setup and Dynamics}
\label{app:agreement-protocol}
The framework we consider for alignment generalizes Aumann agreement~\citep{aumann1976agreeing} to probabilistic $\tuple{\eps, \delta}$-agreement~\citep{aaronson2005complexity} (rather than exact agreement), across $M$ agreement objectives and $N$ agents, \emph{without} the Common Prior Assumption (CPA).
The CPA dates back to at least~\citet{harsanyi1967} in his seminal work on games with incomplete information.
This is a very powerful assumption and is at the heart of Aumann's agreement theorem that two rational Bayesian agents must agree if they share a common prior~\citep{aumann1976agreeing}.
As a further illustration of how powerful the CPA is from a computational complexity standpoint, ~\citet{aaronson2005complexity} relaxed the exact agreement requirement to $\tuple{\eps, \delta}$-agreement and showed that even in this setting, completely independent of how large the state space is, two agents with common priors will need to only exchange $O(1/(\delta\eps^2))$ messages to agree within $\eps$ with probability at least $1-\delta$ over their prior.
However, the CPA is clearly a very strong assumption for human-AI alignment, as we cannot expect that our AIs will always \emph{start out} with common priors with every human it will engage with on every task.
In fact, even between two \emph{humans} this assumption is unlikely!
For other aspects of agreement and how they relate more broadly to alignment, we defer to the Discussion (\S\ref{sec:discussion}) for a more detailed treatment.

In short, \agree-agreement represents a ``best-case'' scenario that is general enough to encompass prior approaches to alignment (cf. Table~\ref{tab:related_work}), such that if something is inefficient here, then it forms a prescription for what to avoid \emph{in practice}, in far more suboptimal circumstances.
As examples of suboptimality in practice, we will consider computational boundedness and noisy messages in \S\ref{ss:results-bounded}, to exactly quantify how the bounds can significantly (e.g. exponentially) worsen.

Dispensing with the CPA, we now make our \agree-agreement framework more precise.
For illustration, we consider two agents ($N=2$), Alice (human) and ``Rob'' (robot), denoted by $\A$ and $\Rob$, respectively.
Let $\{S_j\}_{j\in \M}$ be the collection of (not necessarily disjoint) possible task states for each task $j\in \M$ they are to perform. 
We assume each $S_j$ is finite ($|S_j| = D_j \in \Nat$), as this is a standard assumption, and any physically realistic agent can only encounter a finite number of states anyhow.
There are $M$ agreement objectives, $f_1, \dots, f_M$, that Alice and Rob want to jointly estimate, one for each task:
\begin{equation*}
f_j : S_j \to [0, 1], \quad \forall j \in \M,
\end{equation*}
to encompass the possibility of changing needs and differing desired $\left\{\tuple{\eps_j, \delta_j}\right\}_{j \in [M]}$-agreement thresholds for those needs (which we will define shortly in \eqref{eq:eps-delta}), rather than optimizing for a single monolithic task.
Note that setting the output of $f_j$ to $[0, 1]$ does not reduce generality.
Since $S_j$ is finite, any function $S_j \to \R$ has a bounded range, so one can always rescale appropriately to go inside the $[0, 1]$ domain.

Alice and Rob have priors $\prob^{\A}_j$ and $\prob^{\Rob}_j$, respectively, over task $j$'s state space $S_j$. 
Let $\nu_j \in [0,1]$ denote the \emph{prior distance} (as introduced by~\citet{hellman2013almost}) between $\prob^{\A}_j$ and $\prob^{\Rob}_j$, defined as the minimal $L^1$ distance between any point $x_j \in X_j = \prob^{\A}_j \times \prob^{\Rob}_j$ and any point $p_j \in \mathcal{D}_j
 = \{(p_j, p_j) \mid p_j \in \Delta(S_j)\}$, where $\Delta(S_j)\in \R^{D_j}$ is the probability simplex over the states in $S_j$. 
Formally,
\begin{equation}\label{eq:prior-dist}
\nu_j = \min_{x_j \in X_j, p_j \in \mathcal{D}_j} \|x_j - p_j\|_1, \quad \forall j \in \M.
\end{equation}
It is straightforward to see that there exists a \emph{common prior} $\CP_j \in \mathcal{D}_j$ between Alice and Rob for task $j$ if and only if the task state space $S_j$ has prior distance $\nu_j = 0$.
(Lemma~\ref{lem:common-prior} will in fact show that it is possible to find a common prior with high probability, regardless of the initial value $\nu_j$.)

For every state $s_j \in S_j$, we identify the subset $E_j \subseteq S_j$ with the event that $s_j \in S_j$.
For each task $j \in \M$, Alice and Rob exchange messages\footnote{These messages could be as simple as communicating the agent's current expectation of $f_j$, given (conditioned on) its current knowledge partition. For now, we assume the messages are not noisy, but we will remove this assumption in \S\ref{ss:results-bounded}.} from the power set $\powset{S_j}$ of the task state space $S_j$, as a sequence $m^{1}_j,\dots,m^{T}_j:\powset{S_j} \to [0,1]$.
Let $\Pi_j^{i,t}(s_j)$ be the set of states that agent $i \in \{\A, \Rob\}$ considers possible in task $j$ after the $t$-th message has been sent, given that the true state of the world for task $j$ is $s_j$.
Then by construction, $s_j \in \Pi_j^{i,t}(s_j) \subseteq S_j$, and the set $\{\Pi_j^{i,t}(s_j)\}_{s_j \in S_j}$ forms a \emph{partition} of $S_j$ (known as a ``knowledge partition'').
As is standard~\citep{aumann1976agreeing,aumann1999interactive,aaronson2005complexity,hellman2013almost}, we assume for each task $j$, the agents know each others' initial knowledge partitions $\{\Pi_j^{i,0}(s_j)\}_{s_j \in S_j}$.
The justification for this more broadly~\citep{aumann1976agreeing,aumann1999interactive} is that a given state of the world $s_j \in S_j$ includes the agents' knowledge.
In our setting, it is quite natural to assume that task states for agents coordinating on a task will encode their knowledge.
As a consequence, every agents' subsequent partition is known to every other agent, and every agent knows that this is the case, and so on\footnote{This can be implemented via a ``\emph{common knowledge}'' set, $\mathcal{C}\left(\{\Pi_j^{i,t}\}^{i \in \N}\right)$, which is the finest common coarsening of the agents' partitions~\citep{aumann1976agreeing}.}.
This is because with this assumption, since the agents receive messages from each other, then $\Pi_j^{i,t}(s_j) \subseteq \Pi_j^{i,t-1}(s_j)$.
In other words, subsequent knowledge partitions $\{\Pi_j^{i,t}(s_j)\}_{s_j \in S_j}$ \emph{refine} earlier knowledge partitions $\{\Pi_j^{i,t-1}(s_j)\}_{s_j \in S_j}$.
(Equivalently, we say that $\{\Pi_j^{i,t-1}(s_j)\}_{s_j \in S_j}$ \emph{coarsens} $\{\Pi_j^{i,t}(s_j)\}_{s_j \in S_j}$.)
\emph{Proper} refinement is if for at least one state $s_j \in S_j$, $\Pi_j^{i,t}(s_j) \subsetneq \Pi_j^{i,t-1}(s_j)$, representing a \emph{strict} increase in knowledge.

To illustrate this more concretely, first Alice computes $m^{1}_j\left(\Pi_j^{\A,0}(s_j)\right)$ and sends it to Rob.
Rob's knowledge partition then becomes refined to the set of messages in his original knowledge partition that match Alice's message (since they are now both aware of it):
\begin{equation*}
\begin{split}
\Pi_j^{\Rob,1}(s_j) \;=\;
\bigl\{\, s'_j \in \Pi_j^{\Rob,0}(s_j)
        \mid{} &\;
        m^{1}_j\!\bigl(\Pi_j^{\A,0}(s'_j)\bigr) \\[2pt]
        &=
        m^{1}_j\!\bigl(\Pi_j^{\A,0}(s_j)\bigr)
        \bigr\},
\end{split}
\end{equation*}
from which Rob computes $m^2_j\left(\Pi_j^{\Rob,1}(s_j)\right)$ and sends it to Alice.
Alice then updates her knowledge partition similarly to become the set of messages in her original partition that match Rob's message:
\begin{equation*}
\begin{split}
\Pi_j^{\A,2}(s_j) \;=\;
\bigl\{\, s'_j \in \Pi_j^{\A,0}(s_j)
        \mid{} &\;
        m^{2}_j\!\bigl(\Pi_j^{\Rob,1}(s'_j)\bigr) \\[2pt]
        &=
        m^{2}_j\!\bigl(\Pi_j^{\Rob,1}(s_j)\bigr)
        \bigr\},
\end{split}
\end{equation*}
and then she computes and sends the message $m^{3}_j\left(\Pi_j^{\A,2}(s_j)\right)$ to Rob, etc.

\textbf{\agree-Agreement Criterion:} We examine here the number of messages ($T$) required for Alice and Rob to $\tuple{\eps_j, \delta_j}$-agree across all tasks $j \in \M$, defined as
\begin{equation}\label{eq:eps-delta-alice-bob}
\begin{aligned}
&\Pr\!\Bigl(
   \left|\,
     \Eja{f_j \mid \Pi_j^{\A,T}(s_j)}
     - \Ejr{f_j \mid \Pi_j^{\Rob,T}(s_j)}
   \right|
   \le \eps_j
\Bigr)
\\[4pt]
&> 1 - \delta_j,\quad \forall j \in \M .
\end{aligned}
\end{equation}
In other words, they agree within $\eps_j$ with high probability ($> 1-\delta_j$) on the expected value of $f_j$ with respect to their \emph{own} task-specific priors (not a common prior!), conditioned\footnote{For completeness, note that for any subset $E_j\subseteq S_j$ and distribution $\prob$, $\mathbb{E}_{\prob}[f_j\mid E_j] := \sum_{s_j \in E_j}f(s_j)\prob[s_j \mid E_j] = \dfrac{\sum_{s_j \in E_j}f(s_j)\prob[s_j]}{\sum_{s_j\in E_j}\prob[s_j]}$.} on each of their knowledge partitions by time $T$.

Extending this framework to $N > 2$ agents (consisting of $1 \le q < N$ humans and $N-q \ge 1$ AI agents), is straightforward: we can have their initial, task-specific priors be denoted by $\{\prob_j^i\}^{i \in \N}$, and we can have them $\tuple{\eps_j, \delta_j/N^2}$-agree pairwise so that they globally $\tuple{\eps_j, \delta_j}$-agree.

\section{Modeling Tail Risk}
\label{app:tail-risk}
We note in this section that our $\agree$-agreement framework can also model tail risk/rare events.
For exposition convenience, we use the ``loss'' convention (higher = worse), so the Expected Shortfall (ES)/Conditional Value at Risk (CVaR) at level $\tau\in(0,1]$ uses the upper quantile/tail.
Specifically, for a catastrophe indicator $f_j := Z\in\{0,1\}$ with $\E{f_j} = \Pr[Z=1]=p$, the ES/CVaR at a given level $\tau$ is
\begin{equation*}
\mathrm{ES}^\tau(Z)=\frac{1}{\tau}\int_{1-\tau}^1 q_u(Z)du,
\end{equation*}
where for $Z \sim \mathrm{Bernoulli}(p)$, the quantile is defined as:
\begin{equation*}
q_u(Z) =
\begin{cases}
0, & u \le 1 - p, \\
1, & u > 1 - p.
\end{cases}
\end{equation*}
Therefore,
\begin{equation*}
\begin{split}
\mathrm{ES}^\tau(Z)
&= \frac{1}{\tau} \int_{1-\tau}^{1} \mathbf{1}\{u > 1-p\}\,du \\
&= \frac{1}{\tau}\,\big|(1-\tau,1] \cap (1-p,1]\big| \\
&= \frac{1}{\tau}\min\{\tau,p\} \\
&= \min\left\{1,\frac{p}{\tau}\right\}.
\end{split}
\end{equation*}
Hence, if two models agree on $p$ within $\eps$, their ES values differ by at most $\eps/\tau$. 
For a general bounded loss $f_j := X \in[0,1]$, the~\citet[Eq. 4]{rockafellar2000optimization} representation
\begin{equation*}
\mathrm{ES}^\tau(X)=\inf_{c\in[0,1]}\left(c+\frac{1}{\tau}\mathbb{E}[(X-c)_+]\right)
\end{equation*}
shows that ES is the minimum over expectations of bounded transforms $\psi_c(x)=c+\tfrac{1}{\tau}(x-c)_+$.
Then
\begin{equation*}
\mathrm{ES}^\tau(X)
= \inf_{c \in [0,1]} \, \mathbb{E}\!\left[\psi_c(X)\right].
\end{equation*}

Consequently, for two distributions $P,Q$ over $X \in [0,1]$,
\begin{equation*}
\begin{split}
&\big|\mathrm{ES}^\tau_{P}(X) - \mathrm{ES}^\tau_{Q}(X)\big|\\
&\le \sup_{c \in [0,1]} \big| \mathbb{E}_P[\psi_c(X)] - \mathbb{E}_Q[\psi_c(X)] \big|\\
&\le \frac{1}{\tau} \sup_{c \in [0,1]} 
\big| \mathbb{E}_P[(X - c)_+] - \mathbb{E}_Q[(X - c)_+] \big|.
\end{split}
\end{equation*}
Thus, any $\agree$-agreement bounds controlling expectations of bounded functions directly yield corresponding bounds on ES (scaled by $1/\tau$).

\section{Proofs of Lower Bounds}
\subsection{Proof of Proposition~\ref{prop:lb}}
\begin{proof}
For each task $j \in \M$, let the input tuple to the $N$ agents be
\begin{equation*}
\tuple{x_{1,j},\;x_{2,j},\;\dots,\;x_{N,j}} \;\in\; S_j,
\end{equation*}
where $S_j$ is defined by
\begin{equation*}
\begin{aligned}
S_j \;:=\;
\bigl\{\,\langle x_{1,j},\dots,x_{N,j}\rangle
  \bigm|\;
  &x_{i,j}\in\{(j-1)\,2^{n}+1,\dots\\
  &\qquad j\,2^{n}\},\;\forall i\in[N]\bigr\}.
\end{aligned}
\end{equation*}
Thus, each $x_{i,j}$ is an integer\footnote{One could encode them as binary strings of length at least $n + \lceil \log_2 j\rceil$, but in this proof we do not need the explicit binary representation: the integer \emph{range sizes} themselves suffice to carry out the communication complexity lower bound.} in an interval of size $2^n$ that starts at $(j-1)\cdot 2^n + 1$.
We endow $S_j$ with the uniform common prior $\CP_j$ (which will be necessarily difficult by the counting argument below), and define
\begin{equation*}
f_j\bigl(x_{1,j},\dots,x_{N,j}\bigr) \;=\;
   \frac{\sum_{i=1}^N x_{i,j}}{2^{n+1}}.
\end{equation*}
Observe that $\sum_{i=1}^N x_{i,j}$ is minimally $N\left((j-1)\,2^n + 1\right)$ and maximally $N\,j\,2^n$.
Hence, the image of $f_j$ is contained within
\begin{equation*}
\begin{aligned}
\left[
  \frac{N\bigl((j-1)2^{n}+1\bigr)}{2^{n+1}},
  \frac{N j 2^{n}}{2^{n+1}}
\right]
\\[4pt]
=\;
\left[
  \frac{N(j-1)}{2} + \frac{1}{2^{n+1}},
  \frac{N j}{2}
\right].
\end{aligned}
\end{equation*}
Therefore, for $j \ge 1$, each instance $f_j$ is structurally the same ``shifted'' problem, but crucially \emph{non-overlapping} for each $j \in \Nat$.
So it suffices to show that for each $j$, each instance \emph{individually} saturates the $\Omega\left(N^2\,\log(1/\eps_j)\right)$ bit lower bound, which we will do now:

\emph{Two-Agent Subproblem for $N$ Agents.}
Because \emph{all} agents must $\tuple{\eps_j,\delta_j}$-agree on the value of $f_j$, it follows that in particular, every pair of agents (say $(i,k)$) must have expectations of $f_j$ that differ by at most $\eps_j$ with probability at least $1-\delta_j$.
But for any fixed pair $(i,k)$, we can treat $\left(x_{i,j}, x_{k,j}\right)$ as a two‐agent input in which all other coordinates $x_{\ell,j}$ for $\ell\neq i,k$ are ``known'' from the perspective of these two, or do not affect the difficulty
except to shift the sum\footnote{Equivalently, imagine the other $N-2$ agents are ``dummy'' participants, and we fix their inputs from the perspective of the $(i,k)$ pair.}.
Hence, for each $j$ and each pair $(i,k)$, there is a
two‐agent subproblem.
We claim that these two agents alone already face a lower bound of $\Omega\left(\log(1/\eps_j)\right)$ bits of communication to achieve $\tuple{\eps_j,\delta_j}$-agreement on $f_j$.

Suppose agent $k$ sends only $t<\log_2\left(\tfrac{1-\delta_j}{\eps_j}\right)$ bits to agent $i$ about its input $x_{k,j}$.
Label the $2^t$ possible message sequences by $m=1,\dots,2^t$, with probability $p^m_j$ each.
Since $x_{k,j}$ is uniform in an interval of size $2^n$, then conditioned on message $m$, there remain at least $2^n p^m_j$ possible values of $x_{k,j}$.
Each unit change in $x_{k,j}$ shifts $f_j$ by $1/2^{n+1}$, so even if agent $i$'s estimate is optimal, the fraction of $x_{k,j}$ values producing $|E^{k,t}_j - E^{i,t}_j| \leq \eps_j$ is at most
\begin{equation*}
\frac{2^{n+1}\,\eps_j}{2^n\,p^m_j}\;=\;\frac{2\,\eps_j}{p^m_j}.
\end{equation*}
Hence, the total probability of agreement (over all messages $m$) is bounded by
\begin{equation*}
\sum_{m=1}^{2^t} p^m_j \,\cdot\, \frac{2\,\eps_j}{p^m_j} = 2\,\eps_j\,2^t.
\end{equation*}
If $2\,\eps_j\,2^t < 1-\delta_j$, the agents fail to
$\tuple{\eps_j,\delta_j}$-agree.
Equivalently, $t \ge \log_2\left(\tfrac{1-\delta_j}{\,2\,\eps_j}\right)$.
Since every pair $(i,k)$ needs $\Omega\left(\log\left(1/\eps_j\right)\right)$ bits for each of the $M$ tasks, and there are $\binom{N}{2} = \Theta(N^2)$ pairs, the total cost is
\begin{equation*}
\Omega\left(M\,N^2\,\log\left(1/\eps\right)\right),
\end{equation*}
where $\eps := \min_{j\in \M} \eps_j$, corresponding to the most ``stringent'' task $j$.
\end{proof}

\subsection{Proof of Proposition~\ref{prop:lb:2}}
\begin{proof}
We divide the $M \ge 2$ tasks into two types of payoff functions $f_j$ as follows, each covering the first $\lfloor M/2 \rfloor$ tasks and the last set of $\lceil M/2 \rceil$ tasks, respectively:

\emph{Type I Tasks.}
For the first set of $\lfloor M/2 \rfloor$ tasks, we let the state space be $S_j:=\{{s}_{(j-1)D},{s}_{(j-1)D+1},\dots,{s}_{jD-1}\}$.
Let the $k$-th element of $S_j$ be denoted as $s_{k,j} := s_{(j-1)D+k}$. 
Next, choose a sign vector $b_j\in\{+1,-1\}^{N}$ with $N/2$ plus signs, and set each element of it, $b^i_j$, as follows to define the prior distributions for some $0 < p \le \tfrac{1}{2} - \tfrac{\nu}{4}$ as:
\begin{equation*}
\begin{split}
  &\beta:=\frac{p}{D-2},\qquad
  {\prob}_j^{i}({s}_{0,j})=\tfrac12-\tfrac{b^i_j\nu}{4}-\tfrac{(D-2)\beta}{2},\\
  &{\prob}_j^{i}({s}_{1,j})=\tfrac12+\tfrac{b^i_j\nu}{4}-\tfrac{(D-2)\beta}{2},\;
  {\prob}_j^{i}({s}_{k,j})=\beta\;(k\ge2).
\end{split}
\end{equation*}
If $b^i_j\neq b^k_j$ the agent pair $(i,k)$ has L1 distance $\nu$, and therefore prior distance $\ge \nu$ by definition.

Let $T_j(\omega_j)$ denote the number of bits exchanged on task $j$ when the initial world state is $\omega_j \in S_j$ and the agents follow some message-passing protocol.
We consider the expectation $\mathbb{E}[T_j] := \mathbb{E}_{\omega_j}[T_j(\omega_j)]$ over all initial world states $\omega_j \in S_j$ with respect to the hard prior distributions specified above.
Note that for the purpose of a lower bound, we only need to consider the mismatched agent pairs, since for non-mismatched agents, $T_j \ge 0$, trivially.

Given a task index $j$ and a mismatched agent pair $(i,k)$, let $W^t_j$ denote the total variation distance between the agents' posterior distributions, $\tau^{i,t}_j$ and $\tau^{k,t}_j$, at time $t$:
\begin{equation*}
  W^{t}_j :=\tfrac{1}{2}\bigl\|\tau^{i,t}_j-\tau^{k,t}_j\bigr\|_1,
  \qquad
  W^{0}_j=\nu/2.
\end{equation*}
Define the ``good'' event for task $j$ as
\begin{equation*}
  G_j := \bigl|\,
           E^{i,T}_j
           -E^{k,T}_j
         \bigr|
         \le\eps,
\end{equation*}
which holds with probability at least $1-\delta_j$ by the
$\tuple{\eps_j,\delta_j}$‑agreement condition.
Conditioned on $G_j$, our assumption implies $W^{T_j}_j \le c\nu$ for $c < \tfrac{1}{2} - \tfrac{\delta_j}{\nu}$; on $\overline{G_j}$ we only know the trivial upper bound on total variation of $W^{T_j}_j \le 1$.
Hence,
\begin{equation*}
\mathbb{E}\left[W^{T_j}_j\right] < (1-\delta_j)\cdot c\nu+\delta_j\cdot 1 < c\nu + \delta_j < \nu/2.
\end{equation*}
Since we have that:
\begin{equation*}
T_j = \sum_{t=0}^{T_j-1} 1 \ge \sum_{t=0}^{T_j-1}(W^t_j-W^{t+1}_j) = W^0_j - W^{T_j}_j,
\end{equation*}
by telescoping.
Thus, $\mathbb{E}[T_j]\ge W^0_j-\mathbb{E}[W^T_j] = \Omega(\nu)$, since $\delta_j < \nu/2$.
Hence, each mismatched pair pays $\Omega(\nu)$ bits on task $j$ in expectation.
By a pigeonhole argument, for every initially mismatched agent pair, there exists an initial world state $\omega_j \in S_j$ that attains at least that length transcript length $T_j$; this is the worst‑case for task $j$.
With $\Theta(N^{2})$ such agent pairs, the mismatch cost per task is $\Omega\left(N^{2}\nu\right)$, giving a total worst case bit cost of $\Omega\left(M\,N^2\,\nu\right)$ across the first set of $\lfloor M/2 \rfloor$ tasks.

\emph{Type II Tasks.}
For the remaining set of $\lceil M/2 \rceil$ tasks, we use the hard instance $f_j$ (and its corresponding $N$-tuple state space) in Proposition~\ref{prop:lb} for each task, with a uniform common prior.

Thus, any deterministic transcript that $\langle\eps_j,\delta_j\rangle$‑agrees on \emph{every} task must concatenate $M$ independent sub‑transcripts across the Type I and Type II tasks, giving the final
\begin{equation*}
   \Omega\left(M\,N^{2}\left(\nu +\log\tfrac1\eps\right)\right)
\end{equation*}
bit lower bound.
\end{proof}

\subsection{Proof of Proposition~\ref{prop:lb:3}}
\begin{proof}
We split the $M$ tasks exactly as in Proposition~\ref{prop:lb:2}:
\textbf{Type I:} first $\lfloor M/2\rfloor$ tasks, and
\textbf{Type II:} remaining $\lceil M/2\rceil$ tasks.

Only Type I needs modification; Type II reuses Proposition~\ref{prop:lb} verbatim and costs $\Omega\left(N^{2}\log(1/\eps)\right)$ bits per task.

We consider the following ``uniform-slope'' hard priors for the Type I tasks:
We use the same state-space as in Proposition~\ref{prop:lb:2}'s Type I tasks, namely $S_j:=\{{s}_{(j-1)D},{s}_{(j-1)D+1},\dots,{s}_{jD-1}\}$.
For notational convenience, fix such a task $j$ and relabel its $D:=D_j$ states $S_j=\{s_0,\dots,s_{D-1}\}$, and therefore drop the $j$ subscript.

The priors are defined as follows for agents $i$ and $k$:
\begin{equation*}
\begin{split}
   &{\mathbb P}^{i}(s_m)=\frac{\lambda^{m}}{S},
   \quad
   {\mathbb P}^{k}(s_m)=\frac{\lambda^{-m}}{S'},\\
   &S=\sum_{q=0}^{D-1}\lambda^{q}, \quad
   S'=\sum_{q=0}^{D-1}\lambda^{-q},\quad
   \lambda:=\frac{1+\nu/2}{1-\nu/2}>1.
\end{split}
\end{equation*}
We now show the prior distance for any agent pair $(i,k)$ is $\ge \nu$.
By definition of prior distance, the triangle inequality shows that $\lVert {\mathbb P}^{i}-{\mathbb P}^{k}\rVert_1$ lower bounds the prior distance (for set-valued priors per agent, and for single prior distributions per agent it holds with equality), so it suffices to show that $\lVert {\mathbb P}^{i}-{\mathbb P}^{k}\rVert_1 \ge \nu$.
We have that
\begin{equation*}
\begin{split}
&\lVert {\mathbb P}^{i}-{\mathbb P}^{k}\rVert_{1} := 2\sum_{m:\,{\mathbb P}^{i}(s_{m})>{\mathbb P}^{k}(s_{m})} \left({\mathbb P}^{i}(s_{m})-{\mathbb P}^{k}(s_{m})\right)\\
&\ge\;2\left|{\mathbb P}^{i}(s_{0})-{\mathbb P}^{k}(s_{0})\right|\\
&=2\left|\frac{1}{\sum_{q=0}^{D-1}\lambda^{q}}
    -\frac{1}{\sum_{q=0}^{D-1}\lambda^{-q}}\right| = \frac{2(\lambda-1)(\lambda^{D-1}-1)}{\lambda^{D}-1}\\
&\ge\frac{2(\lambda-1)}{1+\lambda} = \nu,
\end{split}
\end{equation*}
where the last inequality follows from the fact that $\lambda^{D-1} - \lambda \ge 0$ since $\lambda > 1$ and $D \ge 2$, and the last equality directly follows from the definition of $\lambda$.

\paragraph{Canonical chain gap at $t=0$.}
Connectedness of the initial profile implies that for any two states there exists at least one alternating chain of states~\citep[Proposition 2]{hellman2012common}.
In particular, the pair $(s_0,s_{D-1})$ used in the hard prior is linked by a chain \(c^\star=(s_0,s_1,\dots,s_{D-1})\); tightness makes this chain unique and ensures it visits every state exactly once.
We let $\tau^{i,t}$ denote agent $i$'s posterior distribution at time $t$, with $\tau^{i,0} \equiv \mathbb P^i$.
Set
\begin{equation*}
   L_t\;:=\;\left|\log\prod_{(s,s')\in c^\star}
         \frac{\tau^{i,t}(s')}{\tau^{i,t}(s)}
      \;-\;
      \log\prod_{(s,s')\in c^\star}
         \frac{\tau^{k,t}(s')}{\tau^{k,t}(s)}\right|.
\end{equation*}
At $t=0$,
\begin{equation*}
   \prod_{(s,s')\in c^\star}\!\!\frac{\tau^{i,0}(s')}{\tau^{i,0}(s)}
     =\lambda^{D-1},
   \;\;
   \prod_{(s,s')\in c^\star}\!\!\frac{\tau^{k,0}(s')}{\tau^{k,0}(s)}
     =\lambda^{-(D-1)}
\end{equation*}
Thus, we have that
\begin{equation}\label{eq:l0}
\begin{split}
&L_0=2(D-1)\,|\log\lambda|\\
&=2(D-1)\,\left|\log\left(1+\tfrac{\nu}{2}\right)-\log\left(1-\tfrac{\nu}{2}\right)\right|\\
&=2(D-1)\left|\nu + \tfrac{\nu^3}{12}+\dots\right| = \Theta(D\nu),
\end{split}
\end{equation}
where the second to last equality follows by the standard Taylor expansions of $\log(1+x)$ and $\log(1-x)$, since $\nu/2 \le 1$.

\paragraph{Per timestep increment.}
Let the message sent in round $t$ contain $b_t$ bits, and assume the protocol is \textnormal{BBF}$(\beta)$, i.e. for every agent $a$ and all states $s,s'$, the message likelihoods are bounded as such:
\begin{equation*}
   \beta^{-b_t}\;\le\;
     \frac{\Pr[m^{a,t} \mid s, \Pi^{a,t-1}(s)]}
          {\Pr[m^{a,t} \mid s', \Pi^{a,t-1}(s')]}
   \;\le\;\beta^{\,b_t}.
\end{equation*}
Then we will show that the canonical gap satisfies:
\begin{equation}\label{eq:canon-gap-ub}
   |L_t-L_{t-1}|\;\le\;2b_t\log\beta.
\end{equation}

To see this, for convenience we denote $q^a_t(s) :=\Pr[m^{a,t} \mid s, \Pi^{a,t-1}(s)]$ for the message likelihood at time $t$.

Bayes' rule then gives, for any states $s,s'$,
\begin{equation}\label{eq:bayes}
   \frac{\tau^{a,t}(s')}{\tau^{a,t}(s)}
   \;=\;
   \frac{\tau^{a,t-1}(s')}{\tau^{a,t-1}(s)}
   \cdot
   \frac{q^a_t(s')}{q^a_t(s)}.
\end{equation}
Next, define:
\begin{equation}\label{eq:xi}
\Xi^{a}_{t} = \log\!\prod_{m=0}^{D-2}\frac{\tau^{a,t}(s_{m+1})}{\tau^{a,t}(s_m)}
   = \sum_{m=0}^{D-2}\log\frac{\tau^{a,t}(s_{m+1})}{\tau^{a,t}(s_m)} 
\end{equation}
We fix the canonical chain $c^\star$ once for convenience---any path would serve, since we will show that the bound in \eqref{eq:canon-gap-ub} depends only on the \textnormal{BBF}$(\beta)$ likelihood-ratio condition and is independent of the particular chain selected.

It follows from \eqref{eq:bayes} and \eqref{eq:xi} that
\begin{equation*}
\begin{split}
\Xi^{a}_{t} &= \sum_{m=0}^{D-2}
       \biggl[
          \log\frac{\tau^{a,t-1}(s_{m+1})}{\tau^{a,t-1}(s_m)}
          + \log\frac{q^a_t(s_{m+1})}{q^a_t(s_m)}
       \biggr] \\
  &= \underbrace{\sum_{m=0}^{D-2}\log\frac{\tau^{a,t-1}(s_{m+1})}{\tau^{a,t-1}(s_m)}}_{\displaystyle =\;\Xi^{a}_{t-1}}
     \;+\;
     \underbrace{\sum_{m=0}^{D-2}\log\frac{q^a_t(s_{m+1})}{q^a_t(s_m)}}_{\displaystyle\text{telescopes}}
  \\
  &= \Xi^{a}_{t-1}
     \;+\;
     \log\frac{q^a_t(s_{D-1})}{q^a_t(s_{0})}
  \;=\;
     \Xi^{a}_{t-1}
     \;+\;\Delta^{a}_{t}.
\end{split}
\end{equation*}

Because $L_t=|\Xi^{i}_t-\Xi^{k}_t|$, we have by the triangle inequality:
\begin{equation*}
\begin{split}
   &|L_t-L_{t-1}|
   \;=\;
   \bigl|\,|A+B|-|A|\bigr|
   \le|B|,\\
   &A:=\Xi^{i}_{t-1}-\Xi^{k}_{t-1},\;
   B:=\Delta^{i}_t-\Delta^{k}_t.
\end{split}
\end{equation*}
Hence, $|L_t-L_{t-1}|\le|\Delta^{i}_t-\Delta^{k}_t|$.

By \textnormal{BBF}$(\beta)$, each agent alone obeys
\begin{equation*}
  \bigl|\Delta^{a}_t\bigr|
    =\left|\log\frac{q^a_t(s_{D-1})}{q^a_t(s_{0})}\right|
    \le b_t\log\beta,
  \qquad a\in\{i,k\}.
\end{equation*}
Hence, by the triangle inequality,
\begin{equation*}
  \bigl|\Delta^{i}_t-\Delta^{k}_t\bigr|
    \le \bigl|\Delta^{i}_t\bigr|+\bigl|\Delta^{k}_t\bigr|
    \le 2\,b_t\log\beta,
\end{equation*}
giving rise to the desired inequality \eqref{eq:canon-gap-ub}.

\paragraph{Per task cost.}
Let $B^{\text{agree}}$ be the \emph{total number of bits} exchanged by the time the agents agree, and let $B^{\text{cp}}$ be the total bits exchanged when the common prior is reached.
Clearly $B^{\text{agree}}\ge B^{\text{cp}}$, so it suffices to lower bound $B^{\text{cp}}$.

Set $B_T\;:=\;\sum_{t=1}^{T} b_t$, so $B_T$ is the cumulative bit count up to round $T$.
By \citet[Proposition 4]{hellman2012common} we have
$L_T=0$ once the common prior is attained.
Telescoping and \eqref{eq:canon-gap-ub} then gives
\begin{equation*}
\begin{split}
&L_0-L_T
     \;=\;
     \sum_{t=1}^{T}\bigl(L_t-L_{t-1}\bigr)
     \;\le\;\sum_{t=1}^{T}\bigl|L_t-L_{t-1}\bigr|\\
&\le\;2\log\beta\sum_{t=1}^{T} b_t
     =2\log\beta\,B_T.
\end{split}
\end{equation*}

Hence, by \eqref{eq:l0}, any BBF$(\beta)$ protocol must transmit at least $\Omega\left(D\nu\right)$ bits before the priors coincide, and therefore at least that many bits before \agree-agreement.

\paragraph{Aggregating costs.}
There are $\Theta(N^{2})$ mismatched pairs and
$\lfloor M/2\rfloor$ Type I tasks, so the total Type I cost in bits is $\Omega\!\bigl(M\,N^{2}\,D\nu\bigr)$.
Type II contributes $\Omega\!\bigl(M\,N^{2}\log(1/\eps)\bigr)$ bits,
hence the overall lower bound in bits is
\begin{equation*}
   \Omega\left(
        MN^{2}[D\nu+\log(1/\eps)]
      \right),
\end{equation*}
with constant $1/\log\beta$, completing the proof.
\end{proof}

\section{Proof of Lemma~\ref{lem:common-prior}}
\begin{proof}
As before, let $\{\prob_j^i\}^{i \in \N}$, be the priors of the agents.
The ``type profile'' $\tau^t_j$ is the set of the agent's posterior belief distributions over states $s_j \in S_j$ at time $t$.
Thus, at time 0, $\tau^0_j$ will correspond to the prior distributions over the states in the knowledge partition $\Pi_j^{i,  0}$.
Since for each agent its type profile distribution is constant across the states in its knowledge partition $\Pi_j^{i, t}(s_j)$ (as they are indistinguishable to the agent, by definition), then the total size of the type profile at time $t$ is
\begin{equation}\label{eq:type-size}
|\tau^t_j| = \sum_{i = 1}^N \left|\Pi_j^{i, t}\right|.
\end{equation}
We make use of the following result of~\citet[Proposition 2]{hellman2012common}, restated for our particular setting:
Let $\mathcal{C}\left(\{\Pi_j^{i,t}\}^{i \in \N}\right)$ denote the common knowledge set (finest common coarsening) across the agents' knowledge partitions at time $t$.
If the knowledge partitions reach a total size across the $N$ agents that satisfies:
\begin{equation}\label{eq:cp-size}
\sum_{i=1}^N \left|\Pi_j^{i, t}\right| = (N - 1)D_j + \mathcal{C}\left(\{\Pi_j^{i,t}\}^{i \in \N}\right),
\end{equation}
then any type profile $\tau^t_j$ over $\{\Pi_j^{i,t}\}^{i \in \N}$ has a common prior $\CP_j$.
Now, note that $\left|\mathcal{C}\left(\{\Pi_j^{i,t}\}^{i \in \N}\right)\right| \le D_j$ as it forms a partition over the task state space $S_j$, so the set of singleton sets of each element $s_j \in S_j$ has the most components to saturate the upper bound.
Therefore, the desired size $\sum_{i=1}^N \left|\Pi_j^{i, t}\right| \le ND_j$.

Now, starting from an initial type profile $|\tau^0_j|$, the number of proper refinements needed to get to the desired size $\sum_{i=1}^N \left|\Pi_j^{i, t}\right|$ in \eqref{eq:cp-size} is given by \emph{at most}:
\begin{equation*}
\sum_{i=1}^N \left|\Pi_j^{i, t}\right| - |\tau^0_j| + 1 = O\left(ND_j\right).
\end{equation*}
Thus, since\footnote{For example, for $N$ agents that start with maximally unrefined knowledge partitions, $|\tau^0_j| = \sum_{i = 1}^N \left|\Pi_j^{i, t}\right| = \sum_{i=1}^N 1 = N$.} trivially $|\tau^0_j| \ge 0$, then $O(ND_j)$ is the most number of proper refinements we need to ensure there is a common prior with probability 1, by \eqref{eq:cp-size}.
By Lemma~\ref{lem:spanning-tree-refinement}, this amounts to $O(N^2D_j)$ messages in the worst case.
\end{proof}

\section{Proof of Proposition~\ref{prop:disc}}
\label{sec:propdisc-proof}
\begin{proof}
The $N$ agents will communicate with the spanning tree protocol (cf. Lemma~\ref{lem:spanning-tree-refinement}) for each task $j \in \M$, but now with discrete, rather than continuous, messages.
The discretized protocol is as follows:
Let there be a node $F_j$ that is globally accessible to all $N$ agents.
This intermediary is allowed its own prior and will see all messages between the agents (but not their inputs).
Thus, $\Pi_j^{F_j,0}(s_j) = S_j$ for all states $s_j \in S_j$, and $\left\{\Pi^{F_j,t}_j(s_j)\right\}_{s_j \in S_j}$ \emph{coarsens} the knowledge partitions at time $t$ of the $N$ agents, so all of the agents can therefore compute $E^{F_j,t}_j$.
When agent $i$ wants to send a message to its neighbor agent $k$, then agent $i$ sends ``High'' if $E^{i,t}_j > E^{F_j,t}_j + \eps_j/4$, ``Low'' if $E^{i,t}_j < E^{F_j,t}_j - \eps_j/4$, and ``Medium'' if otherwise.
After agent $i$ sends its message to agent $k$, agent $k$ then refines its knowledge partition (and $F_j$ also refines its partition), before agent $k$ sequentially sends its message relative to the current $E^{F_j,t+1}$ to the next agent down the spanning tree.
This process of proper refinement is continued until there is a common prior by Lemma~\ref{lem:common-prior}, which the $N+1$ agents (including $F_j$) then condition on to reach $\tuple{M, N+1, \eps, \delta}$-agreement (hence the $N+1$ factor in the first term to ensure there are enough proper refinements between the $N+1$ agents).
This generalizes~\citet[Theorem 6]{aaronson2005complexity}'s discretized protocol to $N > 2$ agents (and $M > 1$ tasks), which shows that between any \emph{pair} of agents with a common prior, the number of messages needed for them to $\tuple{\eps_j, \delta_j}$-agree remains \emph{unchanged} from the full protocol, where each pair leverages the intermediary agent $F_j$.
Therefore, by applying the spanning tree construction from Lemma~\ref{lem:spanning-tree}, we get the same bound in the discretized case as before of $O(((N+1)g_j^2)/(\delta_j\eps_j)^2)$ messages before the $N+1$ agents \emph{pairwise} $\tuple{\eps_j,\delta_j}$-agree, and therefore $O(((N+1)^5g_j^2)/(\delta_j\eps_j)^2)$ messages until all $\binom{N+1}{2}$ pairs of agents \emph{globally} $\tuple{\eps_j,\delta_j}$-agree, thereby ensuring that the original $N$ agents agree.
Following the rest of the proof of our Theorem~\ref{thm:ub} yields the per-task upper bound of $O\left((N+1)^2D_j + \dfrac{(N+1)^7}{\eps_j^2\delta_j^2}\right)$, where we took the worst-case value of $g_j = O(N + 1)$.
Subsuming lower-order terms in the big-$O$ gives us the stated upper bound.

\paragraph{BBF$(3)$-compliant extension.}
Note that this discretized protocol can be made BBF$(3)$-compliant via the following simple modification: pick any \emph{buffer parameter} $0<\theta\le\tfrac13$ (setting $\theta=\tfrac14$ suffices) and, after the sender has deterministically selected the bucket $B\in\{\text{High},\text{Medium},\text{Low}\}$ according to the thresholds $\pm\eps_j/4$ around $E^{F_j,t}_j$, let the noisy channel transmit the matching 2‑bit codeword with probability $1-2\theta$ and each of the two non‑matching codewords with probability $\theta$, i.e. $\Pr[m_{j}^{i,t}\mid s_j, \Pi^{i,t-1}_j(s_j)] =\theta+(1-3\theta)\,\mathbf 1\{m_j^{i,t}=m_j^{i,t,\star}(B)\}$.
Because every codeword now has probability either $1-2\theta$ or $\theta$ under every state, the message likelihood ratio is always in the range $\left[\tfrac{\theta}{(1-2\theta)},\tfrac{(1-2\theta)}{\theta}\right]$, so the channel is $\operatorname{BBF}(\beta)$ with $\beta=(1-2\theta)/\theta\le3$ when $\theta\ge\tfrac15$. 

We have the agents communicate first, as usual, for $O((N+1)D_j^2)$ messages per task until they reach a common prior and condition on it, by Lemma~\ref{lem:common-prior}.
Thus, the analysis that remains is how many more messages are needed to reach convergence to $\tuple{M, N+1, \eps, \delta}$-agreement.
A round is called \emph{informative} when the sender’s bucket is an outer one (High or Low) and the channel outputs the matching codeword; this occurs with probability at least $(1-2\theta)\,\delta/2$, and in that event $E^{F_j,t}_j$ moves by at least $\eps_j/4$, so the potential $\Psi_t:=\|E^{F_j,t}\|_2^{2}$ increases by at least $(\eps_j/4)^2$.
Hence, $\mathbb E[\Psi_{t+1}-\Psi_t]\ge(1-2\theta)\,\delta_j(\eps_j/4)^2/2$, giving a per-round drift $\kappa_\theta=\Theta\!\bigl((1-2\theta)\,\eps_j^{2}\delta_j\bigr)$.
Define the centered process $Z_t:=\Psi_t-\kappa_\theta t$; then $\{Z_t\}$ is a martingale with one‑step differences bounded by $2$, so the Azuma-Hoeffding inequality can be applied directly to $Z_t$.
Since $\Psi_t\le1$, the additive‑drift (optional‑stopping) theorem implies $\mathbb E[T]\le O\left(1/\kappa_\theta\right)=O\left((1-2\theta)^{-1}\eps_j^{-2}\delta_j^{-1}\right)$ for one pair of agents.
We write $\delta:=\max_{j\in\M}\delta_j$ and, for the high‑probability bound, simply divide this budget evenly: each task gets $\delta/M$ and each of its $\binom{N+1}{2}=O(N^{2})$ pairs gets $\eta:=\delta/(M\binom{N+1}{2})$.  
Azuma-Hoeffding with this $\eta$ yields $O\left(\ln(MN^{2}/\delta)/(1-2\theta)\varepsilon_j^{2}(\delta/M)\right)$ rounds per pair, so a union bound over all $M\binom{N+1}{2}$ pairs leaves total failure probability $\le\delta$.
Plugging the same $\delta/M$ everywhere in the multi‑task bookkeeping of Proposition~\ref{prop:disc} gives
\begin{equation*}
  T
  =O\!\Bigl(
        MN^{2}D
        +\frac{M^{3}N^{7}\ln(MN^{2}/\delta)}
               {(1-2\theta)\varepsilon^{2}\delta}
      \Bigr),
\end{equation*}
valid with probability at least $1-\delta$, with $\eps := \min_{j\in\M}\eps_j$.
\end{proof}

\section{Proofs of Theorem~\ref{thm:bounded} and Corollary~\ref{cor:wannabe-agree}}
\label{app:proofs}

Here we prove both Theorem~\ref{thm:bounded} and Corollary~\ref{cor:wannabe-agree}.
We do this by generalizing~\citet[\S 4]{aaronson2005complexity}'s computational-boundedness treatment from 2 agents to $N$ agents (specifically, $N-q$ agents and $q$ humans that have differing, rather than equal, query costs) and $M$ functions (rather than 1), using a message-passing protocol that combines his smoothed and spanning tree protocols, all \emph{without} the Common Prior Assumption (CPA).
\subsection{Message-Passing Protocol}
\label{app:msp-protocol}
This is the multi-task generalization of~\citet[\S 4.1]{aaronson2005complexity}'s ``smoothed standard protocol'', additionally extended to the multi-agent setting a spanning tree the agents use to communicate their messages.

Let $b_j=\lceil\log_2(\tfrac{C}{\eps_j})\rceil$ be a positive integer we will specify later with a specific constant value $C > 0$, in \eqref{eq:total-bayesian}.
The $N$ computationally bounded agents follow the \agree-agreement algorithm (Algorithm~\ref{alg:agree}), passing $O(b_j)$-bit messages according to the following protocol $\mathcal{P}$:

\textbf{Protocol $\mathcal{P}$ description (for each task \(j\in \M\)):}
\begin{enumerate}
\item \textbf{Current posterior expectation.}
  The \emph{sending agent} $i\in\N$ has at timestep $t-1$ a real value $E^{i,t-1}_j(s_j) \in [0, 1]$, which is its conditional expectation of $f_j \in [0,1]$ given its knowledge partition $\Pi_j^{i,t-1}(s_j)$ and the current task state $s_j \in S_j$.
  (Recall that $E^{i,t-1}_j(s_j) := \Eji{f_j \mid \Pi_j^{i,t-1}(s_j)}$.)
  The knowledge partition $\Pi_j^{i,t-1}(s_j)$ is formed after updating this expectation using Bayes' rule, after having received the earlier message at time $t-2$.

\item \textbf{Draw an integer $r_j$ via a triangular distribution.}
  Agent $i$ picks an integer offset
  \begin{equation*}
    r_j \;\;\in\;\; \{-L_j,\,-L_j+1,\,\dots,\,L_j\}
  \end{equation*}
  according to a (continuous) triangular distribution $\Delta_{\mathrm{tri}}(\,\cdot\,;\,\alpha_j)$ that places its mass in the \emph{discrete} set of values $\{-L_j,\dots,L_j\}$, and has effective width $2\alpha_j$.
  These discrete offsets $r_j$ ensure that the messages will be discrete as well.
  Concretely,
\begin{equation*}
\begin{aligned}
\mathbb{P}[r_j = x]
  &= \frac{L_j - |x|}
          {\displaystyle\sum_{z=-L_j}^{L_j}\!\bigl(L_j - |z|\bigr)}
\\[6pt]
  &= \frac{L_j - |x|}{L_j^{2}},
  \qquad x \in \{-L_j,\dots,L_j\},
\end{aligned}
\end{equation*}
  where $L_j$ is chosen so that $2^{-b_j}L_j = \alpha_j$ to bound the messages, as explained in the next step below.
  Note that the form above is chosen so that the ``peak'' of the discretized triangular distribution is at $r_j = 0$.
  In other words, the form above is maximized in probability when the offset $r_j = 0$ (which means that no noise is added to the messages with the highest probability).

\item \textbf{Form the message with noise.}
  The agent then sets
  \begin{equation*}
     m^t_j\Bigl(\,\Pi_j^{i,t-1}(s_j)\Bigr)
     \;=\;
     \mathrm{round}\left(E^{i,t-1}_j(s_j)\right)
     \;+\;
     2^{-b_j}r_j,
  \end{equation*}
  where $\mathrm{round}\left(E^{i,t-1}_j(s_j)\right)$ denotes rounding $E^{i,t-1}_j(s_j)$ to the nearest multiple of $2^{-b_j}$ (thereby keeping it in the $[0,1]$ interval).
  This ensures that the message $m_j^t$ is itself a multiple of $2^{-b_j}$ (thereby being encodable in $O(b_j)$ bits), and is offset by $\pm \alpha_j$ from $\mathrm{round}\bigl(E^{i,t-1}_j(s_j)\bigr)$, since $|2^{-b_j}r_j| \le 2^{-b_j}L_j = \alpha_j$, by construction.
  Hence, each $m^t_j \in [-\alpha_j, 1+\alpha_j]$.

\item \textbf{Broadcast.}
  This message $m^t_j\bigl(\Pi_j^{i,t-1}(s_j)\bigr)$ is then broadcast (either sequentially or in parallel) to the relevant agents via an edge of the two spanning trees ${SP}^1_j\cup{SP}^2_j$ each of diameter $g_j$, just as in Lemma~\ref{lem:spanning-tree-refinement}, who update their knowledge partitions accordingly to Step 1.
\end{enumerate}

\subsection{Sampling‐Tree Construction and Simulation for Each Task $j\in\M$}
\label{app:sampling-tree}
In our framework, each agent logically refines its knowledge partition $\{\Pi_j^{i,t}(s_j)\}_{s_j \in S_j}$ upon seeing a new message (Step 7 of Algorithm~\ref{alg:agree}).
However, given that the agents are computationally bounded, while refinement is allowed, the issue is with their belief updating.
By Requirement~\ref{req:bounded-cap}, they have no direct ability to sample from the \emph{conditioned} posterior distributions $\tau^{i,t}_j = \prob^i_j(\cdot \mid \Pi_j^{i,t}(s_j))$ at run time, in order to compute the expectation in Step 1 of the protocol in \S\ref{app:msp-protocol}.
In other words, they cannot simply call ``$\mathrm{Sample}\bigl(\prob^i_j\mid \Pi_j^{i,t}(s_j)\bigr)$'' in a black‐box manner. 
Thus, \emph{before} any messages are exchanged, each agent constructs a sampling tree $\mathcal{T}^{i}_{j}$ offline of \emph{unconditional} samples from the priors $\mathbb{P}^{i}_{j}$ (which they are able to do by Subroutine 2 in Requirement~\ref{req:bounded-cap}).
The idea is to precompute enough unconditional draws so that each new message can be \emph{simulated} via ``walking down'' the relevant path in the tree that is consistent with the current message history (including that new message), rather than enumerating or sampling from the newly refined partition directly.

That is the intuition. 
We now explain in detail how each agent can use sampling trees to simulate this protocol in a computationally bounded manner. 
This follows~\citet[\S 4.2]{aaronson2005complexity}'s approach of dividing the simulation into two phases---\emph{(I)~Sampling‐Tree Construction} (no communication) 
and \emph{(II)~Message‐by‐Message Simulation}---but extended here to our multi‐task and multi-agent setting.

\begin{enumerate}
\item[(I)] \textbf{Sampling‐Tree Construction (General $N$‐Agent Version).}
For each task $j\in\M$, and for each agent $i\in \N$, that agent $i$ builds a sampling tree $\mathcal{T}^{\,i}_{j}$ of height $R_j$ and branching factor $B_j$. 
We fix an ordering of the $O\left(R_j\right)$ messages for task~$j$ (so that we know which agent is ``active'' at each level).
Formally, let $\mathrm{ActiveAgent}_j:\{0,\dots,R_j-1\}\to \N$ denote the function specifying which agent sends the message at each round $\ell=0,\dots,R_j-1$. 
For instance, since the spanning tree protocol cycles through the $N$ agents in a consistent order, then $\mathrm{ActiveAgent}_j(0)=i_0$, $\mathrm{ActiveAgent}_j(1)=i_1,\dots$, and so on, up to $R_j$ total rounds.

\begin{itemize}
\item 
  Let $\operatorname{root}^{\,i}_{j}$ be the root node of $\mathcal{T}^{\,i}_{j}$. 
  We say that the \emph{level} of the root node is $0$, its children are at level~$1$, grandchildren at level~$2$, etc., until depth $R_j$.  
  Each node will be labeled by a sample in task $j$'s state space $S_j$, drawn from whichever agent's \emph{unconditional} prior distribution $\prob^k_j(\cdot)$ is active at that level.

\item 
  Concretely, for $\ell=0,\dots,R_j-1$:
  \begin{enumerate}
  \item 
    Let $a_j := \mathrm{ActiveAgent}_j(\ell)$ 
    be the agent whose distribution we want at level~$\ell$ (i.e.\ the agent who sends the $\ell$‐th message). 
    \item 
    Every node $v_j$ at \emph{level~$\ell$} is labeled by some previously chosen sample (if $\ell=0$ and $v_j$ is the root, we can label it trivially or treat $i$'s perspective by a \texttt{do‐nothing} step).  
    \item 
    Each of the $B_j$ children $w\in \mathrm{Children}(v_j)$ at level~$\ell+1$ 
    is labeled with a fresh i.i.d.\ sample drawn from $\prob^a_j(\cdot)$, i.e. from the unconditional posterior of agent $a_j$. 
    \item 
    We continue until level~$R_j$ is reached, yielding a total of 
    \begin{equation}\label{eq:sampling-tree-draws}
      B_j + B_j^2 + \cdots + B_j^{R_j} = \frac{{B_j}^{R_j+1}-1}{B_j-1}-1 = O\left(B_j^{R_j}\right)
    \end{equation}
    newly drawn states from the unconditional prior distributions $\{\prob^k_j\}_{k\in \N}$ at the appropriate levels.
  \end{enumerate}
\end{itemize}

Thus, node labels alternate among the $N$ agents' unconditional draws,  depending on which agent is active at each level (timepoint in message history) $\ell$ in the eventual message sequence for task $j$. 
All of this is done \emph{offline} by \emph{each} agent, requiring no communication among the agents. 
Once constructed, each agent $i$ holds $\mathcal{T}^i_j$ for personal use.

\item[(II)] \textbf{Message‐by‐Message Simulation.}
After building these sampling‐trees $\{\mathcal{T}^i_j\}_{i\in\N}$ (for each task $j$), the $N$ agents enact the protocol in \S\ref{app:msp-protocol} in real time, \emph{but} whenever an agent $i$ needs to ``update its posterior'' after receiving a message, it does \emph{not} sample from $\tau^{i,t}_j = \prob^i_j(\cdot \mid \text{new messages}).$ 
Instead, it uses the precomputed nodes in $\mathcal{T}^i_j$ as follows:

\begin{itemize}
\item \textbf{Initial estimate.}
  At time $t=0$, agent~$i$ approximates $\displaystyle E^{\,i,0}_j(s_j)$ (its prior‐based expectation of $f_j$) by an empirical average of the $B_j$ children of the root node $\operatorname{root}^{\,i}_j \in \mathcal{T}^i_j$. 
  This becomes agent~$i$'s \emph{initial posterior} 
  in the \emph{offline} sense.

\item \textbf{At each round $t=1,\dots,R_j$:}
  \begin{enumerate}
  \item The agent who is ``active'' (i.e.\ is about to send the $t$‐th message) consults \emph{its} sampling‐tree, summing over the relevant subtree that corresponds to the newly received messages 
    $m_j^1,\dots,m_j^{t-1}$, so as to approximate $E^{\,i,t-1}_j(s_j)$.
  \item It picks an integer offset $r_j\in\{-L_j,\dots,L_j\}$ via the discrete triangular distribution (defined in \S\ref{app:msp-protocol}), and then sends the $t$‐th message:
    \begin{equation*}
       m^t_j = \mathrm{round}\bigl(\langle E^{\,i,t-1}_j(\cdot)\rangle_i\bigr) + 2^{-b_j}\,r_j.
    \end{equation*}
  \item The other agents, upon receiving $m^t_j$ via the spanning tree protocol, update node‐weights in their own sampling trees (via the $\Delta$-update equations in \citep[\S4.2]{aaronson2005complexity}).
  This effectively ``follows'' the branch in their sampling trees consistent with $m^t_j$ so they approximate $\Pi^{i,t}_j(\cdot)$ \emph{without} enumerating or sampling from the \emph{conditioned} distribution.
  \end{enumerate}
\end{itemize}
After all $R_j$ messages for task $j$, the agents will have 
approximated the ideal protocol in \S\ref{app:msp-protocol} with high probability (assuming $B_j$ was chosen large enough, which we will give an explicit value for below in the large‐deviation analysis in \S\ref{app:runtime}).
\end{enumerate}

\begin{lemma}[Sampling Tree Time Complexity]\label{lem:sampling-tree-time}
For each task $j$, the time complexity of the sampling tree for all $N$ agents is
\begin{equation}\label{eq:spanning-tree-time}
O\left(B_j^{R_j}T_{N,q}\right).
\end{equation}
\end{lemma}
\begin{proof}
As described, we now assume there are $N$ agents, among which $q$ are humans and $N-q$ are AI agents, each potentially taking different times for evaluating $f_j(\cdot)$ or sampling from the priors (Subroutines 1 and 2 of Requirement~\ref{req:bounded-cap}, respectively):
\begin{equation*}
  T_{\text{eval},H},\quad T_{\text{eval},AI},\quad
  T_{\mathrm{sample},H},\quad T_{\mathrm{sample},AI}.
\end{equation*}
Specifically, a \emph{human} agent $i\in H$ uses time $T_{\text{eval},H}$ to evaluate $f_j(s_j)$ for a state $s_j\in S_j$, and time $T_{\mathrm{sample},H}$ to sample from \emph{another} agent's prior distribution $\prob^k_j(\cdot)$ unconditionally; whereas an \emph{AI} agent $i\in AI$ spends $T_{\text{eval},AI}$ and $T_{\mathrm{sample},AI}$ on the same operations.

\textbf{Sampling Overhead.} By \eqref{eq:sampling-tree-draws}, we do $O\left(B_j^{R_j}\right)$ unconditional draws in total \emph{per agent}.
  Each draw is performed by the \emph{agent building the tree}, but it might sample from \emph{another} agent's distribution.
  Hence the time cost per sample is
  \begin{equation*}
    \begin{cases}
      T_{\mathrm{sample},H}\quad &\text{if the sampling agent is human},\\[4pt]
      T_{\mathrm{sample},AI}\quad &\text{if the sampling agent is an AI}.
    \end{cases}
  \end{equation*}
  We separate the $O\left(q\,B_j^{R_j}\right)$ samples by humans vs. $O\left((N-q)\,B_j^{R_j}\right)$ samples by the AI agents, yielding
  \begin{equation*}
  O\left(q\,B_j^{R_j}\,T_{\mathrm{sample},H} + (N-q)\,B_j^{R_j}\,T_{\mathrm{sample},AI}\right).
  \end{equation*}

\textbf{Evaluation Overhead.} Each sampled state $s_j\in S_j$ may require computing $f_j(s_j)$.  Again, if the \emph{labeling} agent is a human, that cost is $T_{\text{eval},H}$, whereas if an AI agent is performing the labeling, it is $T_{\text{eval},AI}$.
  Consequently,
  \begin{equation*}
    O\left(q\,B_j^{R_j}\,T_{\text{eval},H} + (N-q)\,B_j^{R_j}\,T_{\text{eval},AI}\right)
  \end{equation*}
  suffices to bound all function evaluations across all task $j$ sampling‐trees $\{\mathcal{T}^i_j\}_{i\in\N}$.

\textbf{During the Actual $R_j$ Rounds.} Once messages start flowing, each round requires partial sums or lookups in the prebuilt tree $\mathcal{T}^i_j$.
If agent $i$ is a human in that round, each node update in the subtree will cost either $T_{\text{eval},H}$ or $T_{\mathrm{sample},H}$ (though typically we do not \emph{resample}, so it may be just function evaluations or small indexing).
Since the total offline overhead still dominates, summing over $R_j$ rounds still yields an overall $O\left(B_j^{R_j}\,(T_{\mathrm{sample},\cdot} + T_{\text{eval},\cdot})\right)$ bound, substituting the index $\{\text{H,AI}\}$ depending on the category of the agent.
\newline

Summing the above together gives us the stated upper bound in \eqref{eq:spanning-tree-time}.
\end{proof}

\subsection{Runtime Analysis}
\label{app:runtime}
Recall that the $N$ agents follow the algorithm for \agree-agreement (Algorithm~\ref{alg:agree}), which is a ``meta-procedure'' that takes in any message protocol we have discussed thus far (specifically, the one above).
We now want the agents to communicate with enough rounds $R_j$ such that they can feasibly construct a common prior consistent with their current beliefs, with high probability (Step 9 of Algorithm~\ref{alg:agree}).
We now have to bound $R_j$.

Recall that in the sampling‐tree scenario, agents now no longer have \emph{exact} access to each other's posterior distributions, but rather approximate them by offline sampling---thus they cannot do a \emph{direct, immediate} conditional update (as they could in the unbounded case). 
Indeed, triangular noise does \emph{not} strictly \emph{mask} a surprising message; rather, each agent still \emph{can} receive an improbable message from its vantage. 
However, because these posteriors are only approximate, we must invoke large‐deviation bounds to ensure that with high probability, a message deemed $\gamma$‐likely by the sender but $\leq\gamma/2$ by the \emph{neighboring} receiver 
\emph{actually} appears within some number of messages (in other words, they disagree), forcing a proper refinement in their knowledge partitions. 
We rely on a probabilistic argument that surprises still occur on a similar timescale as in the exact setting, with high probability, thereby prompting a partition refinement:
\begin{lemma}[Number of Messages for One Refinement Under Sampling Trees]
\label{lem:sample-tree-neighbor-msg-refinment}
Suppose two \emph{neighboring} agents $i$ and $k$ have not yet reached \agree-agreement on a given task $j \in \M$.
Therefore, they disagree on some message $m_j^*$ as follows:
\begin{equation}\label{eq:disagree-prob}
\begin{aligned}
\Pr_{\substack{s_j \sim \prob^i_j\\[2pt] r_j \sim \Delta_{\mathrm{tri}}(\alpha_j)}}
   \Bigl[
       \mathrm{round}\!\bigl(E^{i,t-1}_j(s_j)\bigr) + 2^{-b_j} r_j
       = m_j^{*}
   \Bigr]
   \;\ge\; \gamma ,
\end{aligned}
\end{equation}
\emph{namely, Agent \(i\) regards \(m_j^{*}\) as \(\gamma\)-likely.}

\medskip

\begin{equation}\label{eq:disagree-prob2}
\begin{aligned}
\Pr_{\substack{s'_j \sim \prob^k_j\\[2pt] r'_j \sim \Delta_{\mathrm{tri}}(\alpha_j)}}
   \Bigl[
       \mathrm{round}\!\bigl(E^{k,t-1}_j(s'_j)\bigr) + 2^{-b_j} r'_j
       = m_j^{*}
   \Bigr]
   \;\le\; \frac{\gamma}{2} ,
\end{aligned}
\end{equation}
\emph{namely, Agent \(k\) sees \(m_j^{*}\) with probability at most \(\gamma/2\).}
Then the probability of $m_j^*$ failing to be produced in $T = O\left(\frac{\ln\left(\mu_j\right)}{\ln\left(1/\alpha_j\right)}\right)$ consecutive rounds is at most $\mu_j$.
In other words, the neighboring agents will disagree (thereby triggering at least one proper refinement, namely, for agent $k$) with probability at least $1-\mu_j$ after $T = O\left(\frac{\ln\left(\mu_j\right)}{\ln\left(1/\alpha_j\right)}\right)$ consecutive rounds, for $\mu_j > 0$.
\end{lemma}
\begin{proof}
Note that \eqref{eq:disagree-prob} ensures $m_j^*$ has at least probability
$\gamma$ from the \emph{sender’s} perspective, so we can bound how quickly
$m_j^*$ is produced.
Specifically, by \citet[Lemma 15]{aaronson2005complexity}, we know that solely from the sender's vantage (and therefore not assuming a common prior), the probability of $m_j^*$ \emph{not} appearing after $T$ consecutive rounds is given by at most
\begin{equation*}
\lambda_j^{\,T/2}\max\Bigl\{\gamma,\tfrac{1}{B_j}\Bigr\},
\end{equation*}
where $\lambda_j := \frac{4\,e}{\alpha_j}\ln\left(B_j\right)$.
Therefore, it suffices for $T$ to be such that $\lambda_j^{\,T/2}\,\max\{\gamma,1/B_j\} \le \mu_j$. 
Isolating $T$ gives us:
\begin{equation*}
\begin{split}
&\lambda_j^{T/2}
   \;\le\;
   \frac{\mu_j}{\max\{\gamma,\,1/B_j\}}\\
   &\quad\Longrightarrow\quad
   \frac{T}{2}\,\ln\lambda_j
   \;\le\;
   \ln\mu_j-\ln\!\bigl(\max\{\gamma,\,1/B_j\}\bigr).
   \end{split}
\end{equation*}
Hence
\begin{equation*}
\begin{aligned}
T &\le
   \frac{2\!\bigl[\ln\mu_j-\ln\!\bigl(\max\{\gamma,\,1/B_j\}\bigr)\bigr]}
        {\ln\lambda_j}
\\[4pt]
  &=\,
   \frac{2\!\bigl[\ln\mu_j+\ln\!\bigl(1/\max\{\gamma,\,1/B_j\}\bigr)\bigr]}
        {\ln(4e)+\ln(1/\alpha_j)+\ln\!\ln B_j}\, .
\end{aligned}
\end{equation*}
As $\gamma$ is a free parameter, we can obtain a cleaner (albeit looser) bound on $T$ by subsuming suitably large constants.
A natural choice for $\gamma$ is $\gamma = \alpha_j$, since the added noise to each message is at most $\alpha_j < 1/40$ (via Step 3 in \S\ref{app:msp-protocol}) at each round.
Therefore, the maximum additive noise is also the natural threshold for a ``surprising'' message from the receiver's (agent $k$) point of view.
We will also choose the branching factor $B_j$ to be sufficiently large enough that $1/B_j \le \alpha_j$.
Thus, $\max\{\gamma,1/B_j\} = \alpha_j$.
Hence, for $B_j \ge 1/\alpha_j$, trivially $\ln\ln\left(B_j\right) > 0$, and we have that
\begin{equation*}
\begin{aligned}
T
&\;\le\;
  \frac{2\bigl[\ln\mu_j + \ln(1/\alpha_j)\bigr]}
       {\ln(4e) + \ln(1/\alpha_j) + \ln\!\ln B_j}
\\[4pt]
&\;\le\;
  \frac{2\bigl[\ln\mu_j + \ln(1/\alpha_j)\bigr]}
       {\ln(1/\alpha_j)}
\\[4pt]
&=\;
  O\!\left(\frac{\ln\mu_j}{\ln(1/\alpha_j)}\right).
\end{aligned}
\end{equation*}
\end{proof}

We are now finally in a position to bound $R_j$.
\begin{lemma}[Common Prior Lemma Under Sampling Trees]
\label{lem:common-prior-sampling-tree}
Suppose the $N$ agents have not yet reached \agree-agreement on a given task $j \in \M$. 
They will reach a common prior $\CP_j$ with probability at least $1 - \delta_j$, after $R_j = O\left(N^2D_j\frac{\ln\left(\delta_j/(N^2D_j)\right)}{\ln\left(1/\alpha_j\right)}\right)$ rounds.
\end{lemma}
\begin{proof}
We recall from Lemma~\ref{lem:spanning-tree-refinement} that, \emph{if} every
agent had exact access to its posterior distribution, any block of
$O\left(g_j\right)$ messages in the two‐spanning‐trees ordering would pass each agent's ``current expectation'' to every other agent precisely once, guaranteeing that if a large disagreement persists, the \emph{receiving} agent sees an unlikely message and refines its partition.

In the \emph{sampled} (bounded) setting, each agent $i$ approximates its posterior by building an offline tree of $O\left(B_j^{R_j}\right)$ unconditional samples, labeling its nodes by unconditional draws from the respective priors $\{\prob^k_j(\cdot)\}$.
Then, when agent~$i$ must send a message---nominally its exact posterior---it instead \emph{looks up} that value via partial averages in its sampling‐tree.
This is done in a manner consistent with message alternation.
Specifically, in each round, the $\mathrm{ActiveAgent}_j$ function from \S\ref{app:sampling-tree} ensures that every agent's approximate expectation is routed along ${SP}^1_j$ and ${SP}^2_j$ (Step 4 of \S\ref{app:msp-protocol}) once in every block of $O\left(g_j\right)$ messages.

Now, in Lemma~\ref{lem:sample-tree-neighbor-msg-refinment}, we assumed the agents were neighbors.
Thus, in the more general case of $N$ agents, who communicate with spanning trees ${SP}^1_j\cup{SP}^2_j$ each of diameter $g_j$, if there is a ``large disagreement'' between some agent pair $(i,k)$, then from $(i\!\to\!k)$ or $(k\!\to\!i)$'s vantage, the other’s message is improbable.
In the worst case, the agents are $O(g_j)$ hops apart.
Hence, once that message arrives in these $O\left(g_j \frac{\ln\left(\mu_j\right)}{\ln\left(1/\alpha_j\right)}\right) = O\left(N \frac{\ln\left(\mu_j\right)}{\ln\left(1/\alpha_j\right)}\right)$ transmissions, the probability that the receiving agent still sees it as a \emph{surprise} and properly refines its partition is $\ge 1 - g_j\mu_j \ge 1 - N\mu_j$, by a union bound over hops.

These $O\left(N \frac{\ln\left(\mu_j\right)}{\ln\left(1/\alpha_j\right)}\right)$ transmissions constitute one ``block'' of messages that results in at least \emph{one} agents' refinement with high probability.
Finally, we will need $O(ND_j)$ refinements by Lemma~\ref{lem:common-prior} to reach a common prior.
Partition the total timeline into $X=O\left(N D_j\right)$ disjoint ``blocks,'' each block of $O\left(g_j\frac{\ln(\mu_j)}{\ln(1/\alpha_j)}\right)$ messages. 
Let $E_i$ denote the event that ``$\ge 1$ refinement occurs in block $i$''.  
By the single‐block argument, $\prob[E_i]\ge 1 - N\mu_j$.

We want $\prob\Bigl[\bigcap_{i=1}^X E_i\Bigr]$, the probability that \emph{all} $X$ blocks yield at least one refinement.
Using a union bound on complements,
\begin{equation*}
\begin{aligned}
\Pr\!\Bigl[\bigcap_{i=1}^{X} E_i\Bigr]
   &= 1-\Pr\!\Bigl[\bigcup_{i=1}^{X} E_i^{\mathrm c}\Bigr]             \\[4pt]
   &\ge 1-\sum_{i=1}^{X}\Pr\!\bigl[E_i^{\mathrm c}\bigr]              \\[4pt]
   &\ge 1- X\,N\,\mu_j                                                \\[4pt]
   &= 1- N^{2}D_j\,\mu_j .
\end{aligned}
\end{equation*}
Thus, after $X\times O\left(g_j\frac{\ln(\mu_j)}{\ln(1/\alpha_j)}\right) = O\left(N^2 D_j\frac{\ln(\mu_j)}{\ln(1/\alpha_j)}\right)$
rounds, we have converged to a common prior with probability at least 
$1 - N^2 D_j\mu_j$.
Setting $\delta_j := N^2 D_j\mu_j$ gives us the final result.
\end{proof}

Now that we have established that we can converge with high probability $\ge 1 - \delta_j$ to a state where there is a consistent common prior after $R_j = O\left(N^2D_j\frac{\ln\left(\delta_j/(N^2D_j)\right)}{\ln\left(1/\alpha_j\right)}\right)$ rounds, we now introduce an explicit algorithm for the efficient searching of the belief states of the agents.
This is given by Algorithm~\ref{alg:construct}, and finds a common prior via linear programming feasibility of the Bayesian posterior consistency ratios across states.

\begin{algorithm}[ht]
\caption{\textsc{ConstructCommonPrior}$\bigl(\{\Pi^{i,t}_j,\tau^{i,t}_j\}_{i=1}^N\bigr)$}
\label{alg:construct}
\begin{algorithmic}[1]        
\REQUIRE Finite state-space $S_j$ of size $D_j$; for each agent $i\!\in\!\N$:
        partition $\Pi^{i,t}_j=\{C^{i,t}_{j,k}\}_k$  
        and posterior $\tau^{i,t}_j(\,\cdot\,|C^{i,t}_{j,k})$.
\ENSURE  Either a distribution $p_j$ on $S_j$ Bayes-consistent with all
        $\tau^{i,t}_j$, or \textsc{Infeasible}.
\STATE $\displaystyle \Pi_j^{\ast} \gets \bigwedge_{i=1}^N \Pi^{i,t}_j$ \hfill/* intersections */
\STATE Let $\Pi_j^{\ast}=\{Z_1,\dots,Z_r\}$, each $Z_\alpha\subseteq S_j$.
\vspace{0.3em}

\FOR{$\alpha \gets 1$ \TO $r$}
  \STATE create LP variable $p_j(Z_\alpha)\ge 0$, where $p_j(Z_\alpha) := \sum_{s_j \in Z_\alpha}p_j(s_j)$ /* prior mass */
\ENDFOR
\vspace{0.3em}

\FORALL{agent $i\!\in\!\N$ \AND cell $C^{i,t}_{j,k}\!\in\!\Pi^{i,t}_j$}
  \FORALL{$Z_\alpha\subseteq C^{i,t}_{j,k}$}
    \FORALL{$s_j,s_j'\in Z_\alpha$}
      \STATE add constraint
      \[
      \frac{p_j(s_j)}{p_j(s_j')}=
      \frac{\tau^{i,t}_j\!\bigl(s_j\mid C^{i,t}_{j,k}\bigr)}
           {\tau^{i,t}_j\!\bigl(s_j'\mid C^{i,t}_{j,k}\bigr)}
      \]
    \ENDFOR
  \ENDFOR
\ENDFOR
\vspace{0.3em}

\STATE add normalization constraint $\displaystyle\sum_{\alpha=1}^{r}p_j(Z_\alpha)=1$.
\STATE solve the resulting LP for feasibility
\vspace{0.3em}

\IF{a non-negative solution $\{p_j(Z_\alpha)\}$ exists}
  \STATE reconstruct $p_j$ on each $s_j\in Z_\alpha$ via the ratio constraints
  \RETURN $p_j$
\ELSE
  \RETURN \textsc{Infeasible} \hfill/* no single prior fits all posteriors */
\ENDIF
\end{algorithmic}
\end{algorithm} 

We first give proofs of correctness (Lemma~\ref{lem:cp-correctness}) and runtime (Lemma~\ref{lem:cp-runtime}) in the exact setting, before generalizing it to the inexact sampling tree setting (Lemma~\ref{lem:approx-cp}).
\begin{lemma}[Correctness of \textsc{``ConstructCommonPrior''} (Algorithm~\ref{alg:construct})]
\label{lem:cp-correctness}
Let $S_j$ be a finite state space, and for each agent $i\in\N$ let
$\Pi_j^{i,t}$ be a partition of $S_j$ with corresponding posterior
$\tau^{i,t}_j$.
Then the algorithm \textsc{ConstructCommonPrior} returns a probability distribution $p_j$ on $S_j$ that is a Bayes‐consistent common prior for all $\tau^{i,t}_j$ if and only if such a (possibly different) common prior $\CP_j$ exists in principle.
\end{lemma}
\begin{proof}
\textbf{(\(\Longrightarrow\))}\quad
Suppose first that there is some common prior $\CP_j$ over $S_j$ satisfying
\begin{equation*}
\begin{aligned}
&\CP_j\bigl(s_j \mid C_{j,k}^{i,t}\bigr)
  = \tau_j^{i,t}\bigl(s_j \mid C_{j,k}^{i,t}\bigr)
\\
&\qquad \forall\, i\in[N],\; s_j \in C_{j,k}^{i,t}.
\end{aligned}
\end{equation*}
\noindent
That is, the common prior $\CP_j$ agrees with every agent’s posterior on every state in each cell $C_{j,k}^{i,t}$.

In particular, for $s_j,s'_j\in C_{j,k}^{i,t}$,
\begin{equation*}
  \frac{\CP_j(s_j)}{\CP_j(s'_j)} 
  \;=\;
  \frac{\tau^{i,t}_j\bigl(s_j \mid C_{j,k}^{i,t}\bigr)}
       {\tau^{i,t}_j\bigl(s'_j \mid C_{j,k}^{i,t}\bigr)}.
\end{equation*}
Since the meet partition $\Pi_j^*$ refines each $\Pi_j^{i,t}$, those ratio constraints must also hold on every meet‐cell $Z_\alpha\subseteq C_{j,k}^{i,t}$.  
Hence, if we introduce variables for $p_j(Z_\alpha)$ and enforce these pairwise ratio constraints (Algorithm~\ref{alg:construct},
Step 6), the distribution $\CP_j$ serves as a \emph{feasible solution} to that linear system.
Furthermore, the normalization (Step 13) is satisfied by
$\CP_j$.
Consequently, the algorithm will return some valid $p_j$.

\medskip

\noindent
\textbf{(\(\Longleftarrow\))}\quad
Conversely, if the algorithm’s LP is feasible and yields $\bigl\{\,p_j(Z_\alpha)\bigr\}_{\alpha=1}^r$ with $\sum_{\alpha=1}^r p_j(Z_\alpha)=1$, then for any agent~$i$ and cell $C_{j,k}^{i,t}\in \Pi_j^{i,t}$, the meet‐cells $Z_\alpha\subseteq
C_{j,k}^{i,t}$ satisfy
\begin{equation*}
  \frac{p_j(s_j)}{p_j(s'_j)}
  \;=\;
  \frac{\tau^{i,t}_j\bigl(s_j \mid C_{j,k}^{i,t}\bigr)}
       {\tau^{i,t}_j\bigl(s'_j \mid C_{j,k}^{i,t}\bigr)}
  \quad
  \text{for all } s_j,s'_j\in Z_\alpha.
\end{equation*}
Define for each state $s_j\in Z_\alpha$,
\begin{equation*}
  p_j(s_j) = p_j\bigl(Z_\alpha\bigr)\cdot\tau^{i,t}_j\bigl(s_j \mid C_{j,k}^{i,t}\bigr).
\end{equation*}
This $p_j$ is a proper distribution over $S_j$ (since the $p_j(Z_\alpha)$ sum to 1).
Moreover,
\begin{equation*}
  p_j\bigl(s_j\mid C_{j,k}^{i,t}\bigr) = \tau^{i,t}_j\bigl(s_j \mid C_{j,k}^{i,t}\bigr),
\end{equation*}
so $p_j$ is indeed a common prior that \emph{matches} each agent’s posterior distribution.

\medskip

\noindent
\textbf{Remark on the ``true'' prior vs.\ the algorithm's output.}  
If there were a ``true'' common prior $\CP_j$, the distribution $p_j$ returned by the algorithm need not coincide with $\CP_j$ numerically; many distributions can satisfy the same ratio constraints in each cell.
But from every agent's viewpoint, $p_j$ and $\CP_j$ induce the same posterior on all cells, and thus are equally valid as a Bayes‐consistent common prior.

Hence \textsc{ConstructCommonPrior} returns a valid common prior if and only if one exists.
\end{proof}

\begin{lemma}[Runtime of finding a common prior]
\label{lem:cp-runtime}
Given posteriors for $N$ agents over a state space $S_j$ of size $D_j$, a consistent common prior can be found in $O\Bigl(\mathrm{poly}(N\,D_j^2)\Bigr)$ time.
\end{lemma}

\begin{proof}
Each agent’s partition $\Pi_j^{i,t}$ divides $S_j$ into at most $D_j$ cells, so the meet partition
\begin{equation*}
\Pi_j^*=\bigwedge_{i=1}^N \Pi_j^{i,t},
\end{equation*}
has at most $D_j$ nonempty blocks.
Labeling each of the $D_j$ states with its $N$-tuple of cell indices takes $O(N)$ time per state, and grouping (using hashing or sorting) can be done in $O(D_j)$ or $O(D_j\log D_j)$ time. 
Hence, constructing $\Pi_j^*$ takes $\tilde{O}(N\,D_j)$ time (the $\tilde{O}$ subsumes polylogarithmic factors).

Next, for each agent $i\in\N$ and for each cell $C_{j,k}^{i,t}$ in $\Pi_j^{i,t}$, we impose pairwise ratio constraints over states in the same meet-cell contained in $C_{j,k}^{i,t}$. 
In the worst case, an agent's cell may contain all $D_j$ states, leading to $\binom{D_j}{2} = \Theta(D_j^2)$ pairwise constraints for that agent. 
Summing over $N$ agents gives a total of $O(N\,D_j^2)$ constraints.

Standard LP solvers run in time polynomial in the number of variables (at most $D_j$) and constraints ($O(N\,D_j^2)$), plus the bit-size of the numerical data.
Therefore, the overall runtime is $O\Bigl(\mathrm{poly}(N\,D_j^2)\Bigr)$.
\end{proof}

Having proven the runtime and correctness in the \emph{exact} case, we now turn to bounding the runtime in the inexact sampling tree setting.
From here on out, we will define
\begin{equation}\label{eq:cond-time}
T_{N,q,\mathrm{sample}} := q\,T_{\mathrm{sample},H}+(N-q)T_{\mathrm{sample},AI}.
\end{equation}
\begin{lemma}[Approximate Common Prior via Sampling Trees]
\label{lem:approx-cp}
Assume that each agent approximates its conditional posterior $\tau^{i,t}_j(\cdot\mid C_{j,k}^{i,t})$ using its offline sampling tree $\mathcal{T}^{\,i}_j$ (of height $R_j$ and branching factor $B_j$) with per-sample costs $T_{\mathrm{sample},H}$ (for the $q$ human agents) and $T_{\mathrm{sample},AI}$ (for the $N-q$ AI agents). 
Suppose that for each cell $C_{j,k}^{i,t}\subseteq S_j$ and for any two states $s_j,s'_j\in C_{j,k}^{i,t}$, the ratio
\begin{equation*}
\frac{\tau^{i,t}_j(s_j\mid C_{j,k}^{i,t})}{\tau^{i,t}_j(s'_j\mid C_{j,k}^{i,t})}
\end{equation*}
can be estimated via the sampling tree using 
\begin{equation*}
S = O\left(\frac{1}{\eps_j^2}\ln\frac{1}{\delta_j}\right)
\end{equation*}
samples per ratio, so that each is accurate within error $\eps_j$ with probability at least $1-\delta_j$ (we assume sufficiently large ${B_j}$ for this, specifically such that ${B_j}^{R_j} \gtrsim S$ for $R_j$ given by  Lemma~\ref{lem:common-prior-sampling-tree}).
Then the overall time complexity in the sampling-tree setting of $\textsc{ConstructCommonPrior}$ is
\begin{equation*}
O\left(\frac{1}{\eps_j^2}\ln\frac{N\,D_j^2}{\delta_j}\cdot \mathrm{poly}\left(D_j^2T_{N,q,\mathrm{sample}}\right)\right).
\end{equation*}
\end{lemma}
\begin{proof}
For each cell $C_{j,k}^{i,t}$ and each state $s_j\in C_{j,k}^{i,t}$, the agent uses its sampling tree $\mathcal{T}^{\,i}_j$ to compute an empirical estimate $\widehat{\tau}^{i,t}_j(s_j\mid C_{j,k}^{i,t})$ of $\tau^{i,t}_j(s_j\mid C_{j,k}^{i,t})$ by averaging over the leaves of the appropriate subtree.
(Recall that the tree is constructed offline by drawing $O\left(B_j^{R_j}\right)$ i.i.d. samples from the unconditional prior of the active agent at each level, defined by the $\mathrm{ActiveAgent}_j$ function.)
Thus, for a fixed agent $i$, cell $C_{j,k}^{i,t} \subseteq S_j$, and state $s_j \in C_{j,k}^{i,t}$, define the i.i.d. indicator random variables $X_1, \dots, X_m$ by
\begin{equation*}
X_\ell \;=\;
\begin{cases}
1, & \text{if the $\ell$-th sampled state equals } s_j,\\[4pt]
0, & \text{otherwise.}
\end{cases}
\end{equation*}

\noindent
(Each sample is drawn from the leaves of $\mathcal{T}^{\,i}_j$ restricted to the
cell $C_{j,k}^{i,t}$.)
Thus, each $X_\ell$ indicates whether the $\ell$-th sample drawn via the sampling tree lands on the state $s_j$ for agent $i$.
The empirical average
\begin{equation*}
\widehat{\tau}^{i,t}_j(s_j \mid C_{j,k}^{i,t}) = \frac{1}{m} \sum_{\ell=1}^{m} X_\ell
\end{equation*}
serves as an estimate for the true probability $\tau^{i,t}_j(s_j \mid C_{j,k}^{i,t})$.

By the ``textbook'' additive Chernoff-Hoeffding bound, if $m=O\left(\tfrac{1}{\eps_j^2}\ln\tfrac{1}{\delta_j'}\right)$, then
\begin{equation*}
  \Pr\Bigl[\,
    \bigl|\widehat{\tau}^{i,t}_j(s_j\mid C_{j,k}^{i,t})
           -\tau^{i,t}_j(s_j\mid C_{j,k}^{i,t})
    \bigr| \ge \eps_j
  \Bigr] \le \delta'_j.
\end{equation*}
Similarly for $s'_j$, hence the ratio $\widehat{\tau}^{i,t}_j(s_j\mid C_{j,k}^{i,t})
  \big/
  \widehat{\tau}^{i,t}_j(s'_j\mid C_{j,k}^{i,t})$
differs by at most $\pm\eps_j$ with probability $\ge 1-2\delta'_j$, as it is computed from two such independent estimates of $\widehat{\tau}$.  
Taking $\delta'_j=\delta_j/(2\,N\,D_j^2)$ and union‐bounding over all $O(N\,D_j^2)$ ratios yields a net success probability $\ge 1-\delta_j$.

Finally, each ratio uses $O((1/\eps_j^2)\,\ln(1/\delta'_j))$ subtree draws, each draw in the outer \texttt{for} loop in Step 6 of Algorithm~\ref{alg:construct}, incurring $T_{\mathrm{sample},H}$ or $T_{\mathrm{sample},AI}$ cost depending on whether a human or AI agent (of the $N$ total agents) built that portion of $\mathcal{T}^i_j$.
Hence, we replace the $N$ in Lemma~\ref{lem:cp-runtime}'s $O\left(\textrm{poly}\left(N\, D_j^2\right)\right)$ runtime with $q\,T_{\mathrm{sample},H}+(N-q)T_{\mathrm{sample},AI}$, and multiply by the per-ratio sampling overhead of $O((1/\eps_j^2)\,\ln(1/\delta'_j))$.
This gives us the stated runtime.
Thus, \textsc{ConstructCommonPrior} still returns a valid
common prior with probability at least $1-\delta_j$.
\end{proof}

Now we have reached a state where the $N$ agents have found a common prior $\CP_j$ with high probability.
In what follows, note that it does not matter if their common prior is approximate, but rather that the $N$ agents can consistently find \emph{some} common distribution to condition on, which is what Lemma~\ref{lem:approx-cp} guarantees as a consequence.
By Subroutine 2 of Requirement~\ref{req:bounded-cap}, all $N$ agents can sample from the unconditional common prior $\CP_j$ once it is found (Step 12 of Algorithm~\ref{alg:agree}), in total time $T_{N,q,\mathrm{sample}}$.

They will do this and then build their \emph{new} sampling trees of height $R'_j$ and branching factor $B'_j$, post common prior.
We now want to bound $R'_j$:
\begin{lemma}\label{lem:agree-smoothed-standard}
For all $f_j$ and $\CP_j$, the $N$ agents will globally $\tuple{\eps_j, \delta_j}$-agree on a given a task $j \in \M$ after $R'_j = O\left({N^7}/{\left(\delta_j\eps_j\right)^2}\right)$ rounds under the message-passing protocol in \S\ref{app:msp-protocol}.
\end{lemma}
\begin{proof}
By~\citet[Theorem 11]{aaronson2005complexity}, when any \emph{two} agents use this protocol under a common prior, it suffices to take $R'_j = O\left(1/(\delta_j\eps_j^2)\right)$ rounds to reach pairwise $\tuple{\eps_j, \delta_j}$-agreement, which is the same runtime as in the non-noisy two agent case.
We just need to generalize this to $N$ agents who share a common prior $\CP_j$.
We take the same approach as in our Proposition~\ref{prop:disc}, by having an additional node $F_j$ that is globally accessible to all $N$ agents.
The rest of the proof follows their Theorem 11 for any \emph{pair} of agents that use the intermediary $F_j$, so then by our Lemma~\ref{lem:spanning-tree}, under the spanning tree protocol, it will instead require $R'_j = O\left({(N+1)^7}/{\left(\delta_j\eps_j\right)^2}\right)$ rounds for all $\binom{N+1}{2}$ pairs of $N+1$ agents (including $F_j$) to reach \emph{global} $\tuple{\eps_j, \delta_j}$-agreement.
We subsume the lower-order terms in the big-$O$ constant, giving rise to the above lemma.
\end{proof}

Thus, combining Lemma~\ref{lem:agree-smoothed-standard} with Lemma~\ref{lem:sampling-tree-time}, reaching \emph{global} $\tuple{\eps_j, \delta_j}$-agreement with sampling trees will take total time:
\begin{equation}\label{eq:new-tree-time}
O\left({B'_j}^{{N^7}/{\left(\delta_j\eps_j\right)^2}}T_{N,q}\right).
\end{equation}
Let $1-\delta^{\text{find\_CP}}_j$ be the probability of converging to a common prior by Lemma~\ref{lem:common-prior-sampling-tree},
let $1-\delta^{\text{construct\_CP}}_j$ be the probability of constructing a common prior once reaching convergence by Lemma~\ref{lem:approx-cp}, and let $1-\delta^{\text{agree\_CP}}_j$ be the probability of reaching global $\tuple{\eps_j, \delta^{\text{agree\_CP}}_j}$-agreement when conditioned on a constructed common prior, then we have that
\begin{equation}\label{eq:eps-agree-prob}
\begin{aligned}
\Pr\bigl[\eps_j\text{-agreement}\bigr]
   &\;\ge\;
     \bigl(1-\delta^{\text{find\_CP}}_j\bigr)
     \bigl(1-\delta^{\text{construct\_CP}}_j\bigr)\\
     &\bigl(1-\delta^{\text{agree\_CP}}_j\bigr) \\[4pt]
   &\;\ge\;
     1-
     \bigl(
       \delta^{\text{find\_CP}}_j
       +\delta^{\text{construct\_CP}}_j\\
       &+\delta^{\text{agree\_CP}}_j
     \bigr).
\end{aligned}
\end{equation}
Hereafter, we set $\delta_j := \delta^{\text{find\_CP}}_j + \delta^{\text{construct\_CP}}_j + \delta^{\text{agree\_CP}}_j$.

Thus, for a single task $j$, sufficiently large $B'_j$, and $B_j \ge \max\{S^{1/R_j},1/\alpha_j\}$, we have that with probability $\ge 1 - \delta_j$, the $N$ agents will reach $\eps_j$-agreement in time:
\begin{equation*}
\begin{split}
&\widetilde{O}\!\Bigl(
   B_j^{\,N^{2}D_j
          \frac{\ln\!\bigl(\delta^{\text{find\_CP}}_j/(N^{2}D_j)\bigr)}
               {\ln(1/\alpha_j)}}\;
      T_{N,q}                                       \\[4pt]
   &+\;N^{2}D_j\,\operatorname{poly}\!\bigl(D_j^{2}\bigr)\,
      T_{N,q,\mathrm{sample}}
      \;+\;T_{N,q,\mathrm{sample}}                \\[4pt]
   &+\;{B'}_j^{\,\frac{N^{7}}
                   {(\delta^{\text{agree\_CP}}_j\eps_j)^{2}}}\;
      T_{N,q}
\Bigr)
\\[6pt]
&=\;
O\!\Bigl(
      T_{N,q}\Bigl[
          B_j^{\,N^{2}D_j
                 \frac{\ln\!\bigl(\delta^{\text{find\_CP}}_j/(N^{2}D_j)\bigr)}
                      {\ln(1/\alpha_j)}}
          +{B'}_j^{\,\frac{N^{7}}
                       {(\delta^{\text{agree\_CP}}_j\eps_j)^{2}}}
      \Bigr]
\Bigr).
\end{split}
\end{equation*}

since the logarithmic terms (from Lemma~\ref{lem:approx-cp}) subsumed by the $\tilde{O}$ are on the non-exponential additive terms.
This follows from summing the runtime of computing the sampling tree (Lemma~\ref{lem:sampling-tree-time}), the total runtime of the common prior procedure in Lemma~\ref{lem:approx-cp} multiplied by the number of online rounds $R_j$ in the \texttt{while} loop of Algorithm~\ref{alg:agree} (Lines 4-16) given by Lemma~\ref{lem:common-prior-sampling-tree}, and the runtimes of conditioning and the second set of sampling trees in \eqref{eq:cond-time} and \eqref{eq:new-tree-time}, respectively.
We then take $D := \max_{j\in \M} D_j$, $\delta^{\text{find\_CP}}:=\max_{j\in\M}\delta^{\text{find\_CP}}_j$, $\delta^{\text{construct\_CP}}:=\max_{j\in\M}\delta^{\text{construct\_CP}}_j$, $\delta^{\text{agree\_CP}}:=\max_{j\in\M}\delta^{\text{agree\_CP}}_j$, $\alpha := \min_{j \in \M} \alpha_j$, $\eps:= \min_{j \in \M} \eps_j$.
To ensure an overall failure probability of at most $\delta$ we allocate the budget $\delta/M$ \emph{per task}, splitting it equally among the three sub‑routines: $\delta^{\text{find\_CP}}_j=
  \delta^{\text{construct\_CP}}_j=
  \delta^{\text{agree\_CP}}_j=
  \delta/(3M)$.
(A union bound over all $M$ tasks and the three sub‑routines then bounds the total error by $\delta$.)
Finally, let $B:=\max_{j\in\M}\max\{B_j,B'_j\}$ so that a single base subsumes all exponential terms.  
Multiplying the worst‑case per‑task time by $M$ tasks yields the global bound of Theorem~\ref{thm:bounded}.

To prove Corollary~\ref{cor:wannabe-agree}, it suffices to explicitly bound $B'_j$.
This is because for each \emph{individual} task $j$, a ``total Bayesian wannabe'', \emph{after} conditioning on a common prior $\CP_j$, as defined in Definition~\ref{def:total-wannabe}, corresponds to a single ``Bayesian wannabe'' in~\citet[pg. 19]{aaronson2005complexity}'s sense.
Thus, by their Theorem 20, for $\eps_j/50 \le \alpha_j \le \eps_j/40$, it suffices to set on a per-task basis,
\begin{equation}\label{eq:total-bayesian}
B'_j := O\left((11/\alpha_j)^{{R'_j}^2}/\rho_j^2\right), \quad 
b_j := \left\lceil \log_2 \tfrac{R'_j}{\rho_j \alpha_j} \right\rceil + 2,
\end{equation}
where $R'_j$ is the value for $N$ agents that we found in Lemma~\ref{lem:agree-smoothed-standard}.
Taking $\rho := \min_{j \in \M} \rho_j$ maximizes the bound, and scaling by $M$ gives rise to what is stated in Corollary~\ref{cor:wannabe-agree}.

\section{Proof of Proposition~\ref{prop:needle}}
\begin{proof}
\textbf{Hard priors.}
For a state space $S=\{s_{0},s_{1},\dots ,s_{D-1}\}\;(D\ge 3)$, define
\begin{equation*}
\begin{aligned}
{\mathbb P}^{i}(s_{0})  &= 0,
&
{\mathbb P}^{k}(s_{0})  &= \nu/2, \\
{\mathbb P}^{i}(s_{m})  &= \frac{1}{D-1},
&
{\mathbb P}^{k}(s_{m})  &= \frac{1-\nu/2}{D-1}\quad(m\ge 1).
\end{aligned}
\end{equation*}
Then $\|{\mathbb P}^{i}-{\mathbb P}^{k}\|_{1}=\nu$, which means the prior distance $\ge \nu$ by the triangle inequality.

\paragraph{Payoff and targets.}
Let $f(s)=\mathbf 1[s=s_{0}]$; thus $\mathbb E_{i}[f]=0,\;
 \mathbb E_{k}[f]=\nu/2$.
Set $\eps=\nu/8,\;\delta=1/4$.

\textbf{One agent, one tree.}
It suffices to consider only agent $k$.
Let the sampling tree have $L$ leaves labeled $s_{1},\dots ,s_{L}$ and define the Bernoulli indicators $X_{\ell}=\mathbf 1[s_{\ell}=s_{0}]$.
The empirical expectation sent in the \emph{first} message is
\begin{equation*}
  \widehat{E}[f]
     \;=\;
     \frac1L \sum_{\ell=1}^{L} X_{\ell}.
\end{equation*}

If none of the $L$ samples equals $s_{0}$, then $\widehat{E}[f]=0$ and the absolute error is $|\widehat{E}[f]-\mathbb E_{k}[f]|=\nu/2\ge 4\eps$; the protocol necessarily fails.  
The probability of that event is
\begin{equation*}
  \Pr[\text{no }s_{0}]
     =(1-\nu/2)^{L}
     \ge 1-(\nu/2)L
\end{equation*}
To make it $\le\delta=1/4$, we need $1-(\nu/2)L\le 1/4$, then $L \ge 3/(2\nu)$.
Because the sampling tree has to be at least of size $L$, we complete this part.

\textbf{Many tasks \& agents.}
Embed an independent copy of the hard task in each of $M$ task indices, assign $P^i$ to $N/2$ of the agents and $P^k$ to the other $N/2$ agents, and sum the costs to get
$\Omega(M\nu^{-1}\,T_{N,q,\mathrm{sample}})$ time units once the per‑sample cost of Requirement~\ref{req:bounded-cap} is applied.
\end{proof}